\documentclass[12pt]{article}
\usepackage[margin=1in]{geometry}
\usepackage[utf8]{inputenc}

\usepackage[english]{babel}
\usepackage[round]{natbib}
\usepackage[font=footnotesize]{caption}
\usepackage{mathtools,amssymb,amsthm}
\usepackage{subcaption}
\usepackage{pifont}
\usepackage{booktabs}

\usepackage{algorithm}
\usepackage{algorithmic}
\usepackage{wrapfig}
\usepackage{float}
\usepackage[toc,page]{appendix}

\usepackage[dvipsnames]{xcolor}
\definecolor{citeCol}{rgb}{0.5412, 0.2118, 0.0588}
\usepackage[colorlinks=true, citecolor=citeCol, linkcolor=citeCol, urlcolor=citeCol]{hyperref}

\usepackage[capitalize]{cleveref}
\usepackage{bm}
\usepackage{enumitem}

\usepackage{amsmath, amssymb, amsthm,mathtools,dsfont}
\usepackage{bbm}
\usepackage{natbib}
\usepackage{blkarray}
\usepackage{cancel}
\usepackage{enumitem}
\usepackage{euscript}
\usepackage{float}
\usepackage{bm}
\usepackage{mathrsfs}
\usepackage{microtype}
\usepackage{ragged2e}
\usepackage{sectsty}
\usepackage{thmtools}
\usepackage{tikz}
\usepackage[Symbolsmallscale]{upgreek}

\usepackage{microtype}
\usepackage{subcaption}

\usepackage{hyperref}

\usepackage{fancyhdr}

\newcommand{\cE}{\mathcal{E}}

\newcommand{\cI}{\mathcal{I}}

\newcommand{\cL}{\mathcal{L}}

\newcommand{\cN}{\mathcal{N}}

\newcommand{\cV}{\mathcal{V}}

\newcommand{\cX}{\mathcal{X}}

\newcommand{\PP}{\mathbb{P}}

\newcommand{\RR}{\mathbb{R}}

\newcommand*{\norm}[1]{\left\|#1\right\|}
\newcommand*{\triplenorm}[1]{{\left\vert\kern-0.25ex\left\vert\kern-0.25ex\left\vert #1
    \right\vert\kern-0.25ex\right\vert\kern-0.25ex\right\vert}}

\DeclareMathOperator{\supp}{supp}

\newcommand{\R}{\mathbb{R}}

\renewcommand{\phi}{\varphi}
\newcommand{\eps}{\varepsilon}

\newcommand*{\E}{\mathbb E}

\DeclareMathOperator{\var}{Var}

\DeclareMathOperator{\tr}{tr}

\newcommand*{\defeq}{\coloneqq}
\newcommand*{\rd}{\mathrm{d}}
\newcommand*{\D}{\mathrm{d}}

\newcommand*{\dd}{\, \rd}
\DeclareMathOperator*{\argmin}{arg\,min}
\DeclareMathOperator*{\argmax}{arg\,max}

\DeclareMathOperator*{\KL}{KL}

\newcommand\bs[1]{\boldsymbol{#1}}

\newcommand\mb[1]{\mathbf{#1}}

\newcommand\mc[1]{\mathcal{#1}}

\newcommand\msf[1]{\mathsf{#1}}

\newcommand{\pdata}{p_X}
\newcommand{\cXbar}{\overline{\mc X}}
\newcommand{\Ndata}{N_X}
\newcommand{\dstar}{\mathsf k}

\newcommand{\Id}{I}
\newcommand{\tsigma}{\tilde\sigma}
\newcommand{\xbar}{\bar x}
\newcommand{\pbar}{\pdata}

\newcommand{\tri}[1]{\langle #1 \rangle}

\newcommand{\prn}[1]{(#1)}

\newcommand{\abs}[1]{|#1|}

\newcommand{\epsBD}{\epsilon_{\rm BD}}
\newcommand{\tepsBD}{\tilde\epsilon_{\rm BD}}
\newcommand{\ta}{\mathsf{a}}
\newcommand{\prior}{\Pi_0}
\newcommand{\post}{\Pi}
\newcommand{\ppost}{\overline\Pi}

\renewcommand{\epsilon}{\varepsilon}

\definecolor{blue}{rgb}{0,.1, .6}
\definecolor{royalblue}{rgb}{0.2549, 0.4118, 0.8824}

\definecolor{mypink}{rgb}{0.858, 0.188, 0.478}

\theoremstyle{plain}
\newtheorem{theorem}{Theorem}[section]
\newtheorem{proposition}[theorem]{Proposition}

\newtheorem{lemma}[theorem]{Lemma}
\newtheorem{corollary}[theorem]{Corollary}
\theoremstyle{definition}
\newtheorem{definition}[theorem]{Definition}

\theoremstyle{remark}

\begin{document}
\vspace*{0.15in}

\begin{center} {\LARGE{{Blind denoising diffusion models and  \\ the blessings of dimensionality}}}

{\large{
\vspace*{.3in}
\begin{tabular}{cccc}
Zahra Kadkhodaie$^{1,*}$, Aram-Alexandre Pooladian$^{2,*}$ \\ Sinho Chewi$^{3}$, and Eero P. Simoncelli$^{1,4}$\\
\end{tabular}
{
\vspace*{.1in}
\begin{tabular}{c}\\
				$^1$Flatiron Institute, Simons Foundation.\\
				$^2$Foundations of Data Science, Yale University.\\
				$^3$Department of Statistics and Data Science, Yale University.\\
				$^4$Ctr. for Neural Science \& Courant Institute, New York University.\\
 				\small{\texttt{zkadkhodaie@flatironinstitute.org}}\\
                		\small{\texttt{\{aram-alexandre.pooladian,sinho.chewi\}@yale.edu}}\\
				\small{\texttt{esimoncelli@flatironinstitute.org}}\\
\end{tabular} 
}
}}
\vspace*{.1in}

\today

\end{center}

\vspace*{.1in}

\footnotetext{*Equal contribution.}

\abstract{
Denoising diffusion models (DDMs) are state-of-the-art methods for learning densities from data 
{across numerous domains,}
yet many aspects of the training and sampling pipeline remain poorly understood. {In particular}, noise conditioning {requires} practitioners to incorporate contrived unprincipled noise embeddings into neural network architectures and to use ad hoc noise schedules for sampling. To {address} these {drawbacks}, we provide a complete theory for \emph{blind denoising diffusion models} (BDDMs): a variant of DDMs where the noise amplitude is not passed into the neural network during training or sampling, {obviating the need for the aforementioned design choices.}
We justify the correctness of BDDMs as a sampling algorithm under an assumption of low intrinsic dimensionality of the underlying data distribution relative to the ambient dimension. This assumption arises through the introduction of the Bayesian problem of estimating noise levels from a single noisy sample, which might be of independent interest. We empirically compare the performance of BDDMs to standard DDMs, showcasing the benefits of an \emph{adaptive} scheme which is rigorously justified by our analysis.\looseness-1
}

\medskip{}

\section{Introduction}
Denoising diffusion models (DDMs) have emerged as the dominant methodology for generative modeling of images and solving inverse problems \citep{sohl2015deep, ho2020denoising,Song+21SDE,Chung+23Posterior}. In contrast to previous machine learning methods (e.g., generative adversarial networks \citep{WassersteinGAN}, normalizing flows \citep{Gra+19FFJORD}, or variational autoencoders \citep{KinWel14VB}), diffusion models are based on corresponding noise corruption and denoising processes \citep{Fol85}. In brief, the goal of DDMs is to sample from a data distribution $p_X$ through discretizations of stochastic differential equations (SDEs) of the form
\begin{align}\label{eq:non-blind_intro}
    \dd X_t = a_t \widehat s_\theta(X_t, \sigma_t) \dd t + \sqrt{2a_t}\dd B_t\,,
\end{align}
where $\widehat s_\theta : \R^d \times \R_+ \to \R^d$ is a parametric neural network trained on samples from $p_X$ to approximate the score, $\nabla \log p_\sigma(x_\sigma)$, $\dd B_t$ is standard Brownian motion, $(a_t)_{t \geq 0}$ are diffusion coefficients, and $(\sigma_t)_{t\geq0}$ is a sequence of noise levels. The barriers to efficient and accurate sampling are then (a) obtaining a good approximation of the score $\widehat s_\theta$ and (b) choosing sequences  $(\sigma_t)_{t\geq 0}$ and $(a_t)_{t \geq 0}$ that \emph{predict} the noise level on the samples when discretizing \eqref{eq:non-blind_intro}.

The first hurdle is resolved through an application of Tweedie's formula \citep{Robbins1956Empirical, Miyasawa1961Empirical}, which states that minimum mean square error (MMSE) denoiser (the mean of the posterior) can be re-written as the variance-scaled score function \looseness-1
\begin{align}\label{eq:tweedie}
     r_{\sigma_t}^\star(y) \defeq \sigma_t^2\nabla \log p_{\sigma_t}(y) = 
     \E[X\!\mid\! Y=y] - y\,,
\end{align}
where the {conditional} expectation is {over $X\sim \pdata$ given $Y \sim \cN(X,\sigma_t^2 I)$.} 
With \eqref{eq:tweedie}, the score can be learned by minimizing a regression-like objective \citep{Hyv05ScoreMatch, Raphan11least}. It is worth emphasizing that the parametric neural network takes both a noisy image $y$ and its noise level $\sigma_t$ as inputs, the latter implemented through the use of noise embeddings \citep{Song+21SDE}. \looseness-1

As for the second hurdle, the noise schedule and diffusion coefficients are often ad hoc and empirically tuned.
For instance,~\citet{karras2022elucidating} propose the schedule
\begin{align}\label{eq:edm_schedule}
    \sigma_k = {\Bigl(\sigma_{\max}^{1/7} - \frac{k}{N}\,(\sigma_{\max}^{1/7} - \sigma_{\min}^{1/7})\Bigr)^7}\,.
\end{align}
Perhaps surprisingly, we observe empirically that these noise schedules do not precisely predict the noise level on samples throughout the reverse process (see Figure~\ref{fig:mismatch_schedule}). 
So, although the dynamics in \eqref{eq:non-blind_intro} are mathematically justified by the theory for forward-reverse diffusions, these justifications break down at the implementation stage. This breakdown is due to an inconsistency in the discretized reverse process: the score model $\widehat s_\theta$ takes in two arguments $(x_\sigma,\sigma)$ in which the second argument $\sigma$ obtained from the schedule does not accurately predict the $\sigma$ of the noisy sample $x_\sigma$ accurately (see \cref{fig:mismatch_schedule}). 

One way to resolve this inconsistency is to entirely remove the second argument (i.e., the noise level) and allow the model to infer it from the noisy sample. Such a construction is called a ``{blind}'' model in the denoising signal processing literature \citep{LiuTanOku13Blind, zhang2017beyond, MohanKadkhodaie19b}. Despite the absence of an explicit noise level, these models have been shown to achieve denoising performance rivaling that of the non-blind models.

In this work, we study denoising diffusion models that employ blind denoisers, which we call \emph{blind denoising diffusion models} (BDDMs).
BDDMs are trained \emph{without} conditioning on the noise level of the image and also require no explicit noise schedule $(\sigma_k)_{k \geq 0}$ during inference, nor an external diffusion rate $(a_k)_{k \geq 0}$. Such a sampling scheme was initially proposed by \citet{kadkhodaie2020solving}: letting $\widehat f_\theta$ denote a trained blind denoiser (see \eqref{eq:empirical_loss}), their proposed scheme essentially amounts to
\begin{align}\label{eq:bddm_sampler_intro}
\begin{split}
        Y_{(k+1)h} &= Y_{kh} + h\bigl(\widehat f_\theta(Y_{kh}) - Y_{kh}\bigr) + \sqrt{2h \beta \widehat\sigma_k^2 }\, \xi_k\,
\end{split}
\end{align}
where $\widehat\sigma_k^2 \defeq \|\widehat f_\theta(Y_{kh}) - Y_{kh}\|^2/d$ is an estimate of noise variance computed at each iteration, $h =\Delta t$ is the discretization, $ \beta \in (0,1]$ is a constant diffusion coefficient, and $\xi_k \sim \cN(0,I)$. The rightmost plot in Figure~\ref{fig:mismatch_schedule} demonstrates that unlike traditional denoising schedules (e.g., log-uniform or the schedule given by \eqref{eq:edm_schedule}), this implicit choice exactly tracks the noise level of the image, resulting in minimal mismatch along the denoising process.

\subsection*{Contributions}
\begin{figure}[t]
    \centering
    \includegraphics[width=.95\linewidth]{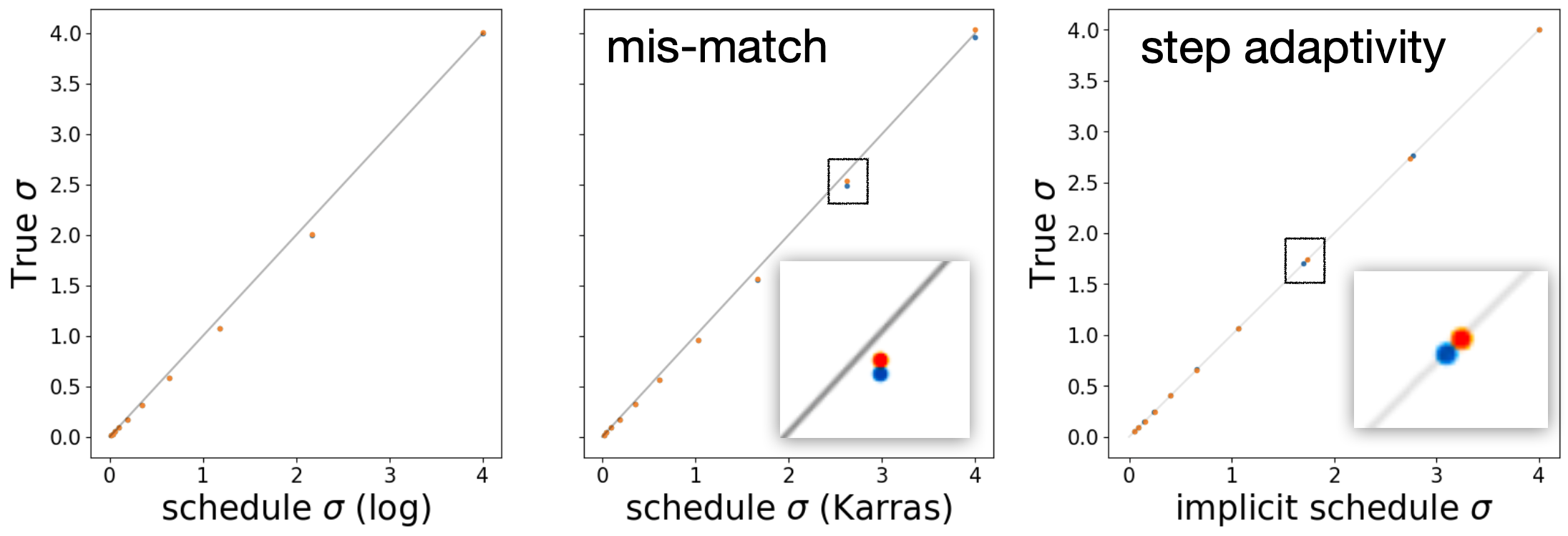}     
    \caption{
    Empirical comparison of scheduled noise level $\sigma_t$ and estimated noise level $\sigma^\star$ along trajectories for the CelebA dataset. The DDPM reverse process outpaces a log schedule \textbf{(left)}, or \eqref{eq:edm_schedule} proposed by \citet{karras2022elucidating} \textbf{(middle)}, leading to an excess error manifested in the loss of detail in samples. In contrast, our {BDDM tracks the implicit schedule $\sigma_t$ via $\hat{\sigma}_k = {\Vert x_k - f(x_k)\Vert}/{\sqrt{d}}$}
     \textbf{(right)}.\looseness-1
    }
    \label{fig:mismatch_schedule}
\end{figure}
Our main contribution is to provide {the first} end-to-end theoretical justification for BDDMs.
In particular, we prove that sampling schemes like \eqref{eq:bddm_sampler_intro} sample from the true data distribution $p_X$ (see Theorem~\ref{thm:thm_main}). Our main assumption is that $p_X$ has \emph{low intrinsic dimensionality} with respect to the ambient space (see Definition~\ref{def:intdim}). 

{Along the way, we develop a number of useful insights. Namely, we derive an analytical formula for the evolution of noise level along the reverse trajectory---see~\eqref{eq:noise_formula}---and provide statistical theory justifying why it can be estimated along the reverse diffusion trajectory.}
Moreover, while our theory supports \emph{any} diffusion coefficient (see \eqref{eq:empirical_sde}), our discretization analysis in Section~\ref{sec:theory_discretization} suggests that schemes of the form \eqref{eq:bddm_sampler_intro} are {possibly optimal}---see Theorem~\ref{thm:disc_error} and Corollary~\ref{cor:sampling_cor}.

{Some recent works have investigated blind denoising for DDMs and related flow-based models, however their conclusions remain far from comprehensive or rigorous \citep{sun2025noise,wang2025equilibrium}.}

\subsubsection*{Notation}
We denote the centered Gaussian with variance $\sigma^2 > 0$ as $\gamma_{\sigma^2} = \cN(0,\sigma^2 I)$. Throughout, we denote the data distribution as $\pdata$, and the noisy data distribution(s) as $p_\sigma = \pdata * \gamma_{\sigma^2}$. We let $\dd B_t$ denote standard Brownian motion. The Kullback--Leibler divergence between probability densities $p,q$ is written ${\rm KL}(p\|q) = \int \log(p(x)/q(x))\dd p(x)$.

\section{Preliminaries on blind denoisers}
\subsection{Training}\label{sec:bg_blind}
Unlike standard denoisers in the diffusion model literature, blind denoisers are not given the noise level during training. For a neural network $f_\theta : \R^d \to \R^d$, a \emph{blind denoiser} is trained using the following objective \citep{zhang2017beyond, GnaCha20OneSize, MohanKadkhodaie19b}:
\begin{align}\label{eq:empirical_loss}
    \widehat f_\theta\! = \argmin_\theta \frac{1}{n} \sum_{i=1}^n  \underset{\substack{\sigma\sim\prior \\ z\sim \gamma}}{\E} \|\bs x_i - f_\theta(\bs x_i + \sigma z)\|^2 
\end{align}
where $\{\bs x_i \}_{i=1}^n \sim \pdata$ are the clean training samples, 
$\gamma$ is the standard Gaussian in $\R^d$,
and $\prior$ is a prior distribution over noise levels in the range $[\sigma_T,\sigma_{0}]$, with 
$0 < \sigma_{T} < \sigma_{0} < +\infty$
(since we are interested in the reverse process, $\sigma_0$ is the largest noise level, and $\sigma_T$ is the smallest). In practice, we minimize~\eqref{eq:empirical_loss} using (batch) stochastic gradient descent, see Algorithm~\ref{alg:train} for pseudocode.

\subsection{Sampling}\label{sec:bg_blind_sample}
To sample, we initialize $Y_{0} \sim \cN(0,\sigma_0^2I)$ and consider various discretizations of 
\begin{align}\label{eq:empirical_sde}
    \dd Y_t = (\widehat f_\theta(Y_t) - Y_t)\dd t + \sqrt{2a_t}\dd B_t\,;
\end{align}
This is the general continuous analogue of \eqref{eq:bddm_sampler_intro} with an arbitrary sequence $(a_t)_{t\geq 0}$. For example, a simple Euler discretization gives, for a constant step size $h > 0$ and $\xi_k \sim \cN(0,I)$
\begin{align*}
    Y_{(k+1)h} = Y_{kh} + h\bigl(\widehat f_\theta(Y_{kh}) - Y_{kh}\bigr) + \sqrt{2a_k h}\,\xi_k\,.
\end{align*}
See Algorithm~\ref{alg:inf} for the inference algorithm. We additionally discuss the exponential (Euler) integrator in Appendix~\ref{app:exp_euler}, as it plays a role in our analysis.

\subsection{Related work}
BDDMs were introduced by \citet{kadkhodaie2020solving,kadkhodaie2021stochastic}, where a blind deep neural network denoiser \citep{zhang2017beyond} was used to sample from $\pdata$. The capabilities of BDDMs for sampling and solving linear inverse problems were demonstrated empirically and their generalization (with respect to training set size) was examined using U-Net architectures \citep{ronneberger2015u} by \citet{Kad+24Generalization}. 

Closest to our work is that of \citet{sun2025noise}, where the authors show that noise conditioning is not required for sampling. They verify this empirically on dynamics which require explicit schedules (e.g., DDPM, DDIM, EDM), and thus do not fully investigate adaptive algorithms as we do. 
On the theoretical front, they make initial progress by studying the simple cases where $\pdata$ is a single Dirac mass or mixture of Dirac masses.  We depart from these simple scenarios by showing that {blind denoisers succeed due to low intrinsic dimensionality of the underlying data distribution}. We provide rigorous derivations under this assumption, and provide numerous other technical and empirical contributions of interest to diverse communities. 

In other works, neural networks have been trained to learn dynamics without time inputs~\citep[e.g.,][]{du2019implicit,wang2025equilibrium}. We believe our work provides a basis for understanding how these training dynamics are possible, accompanied by theoretical guarantees.

\section{Theoretical contributions}\label{sec:theory}
This section contains our theoretical contributions, which rigorously justify the use of BDDMs as generative models. All proofs are provided in the appendix.
\subsection{The optimal blind denoiser}\label{sec:theory_optimal}
The first step is to understand the population minimizer of the blind denoising problem of \eqref{eq:empirical_loss}: Given infinite data and perfect optimization, what is learned? Since the noise level is not known, it should be treated as a random variable (as opposed to a known parameter). From a Bayesian perspective, the optimal solution then comes from marginalizing over this random variable. The following theorem makes this precise and indicates that the optimal blind denoiser can be expressed as a \emph{conditional average} of score functions.
\begin{proposition}\label{prop:bd_minimizer}
The population minimizer of \eqref{eq:empirical_loss} is 
\begin{align*}
    f^\star : y\mapsto y + \int \sigma^2\, \nabla \log p_\sigma(y) \dd \post(\sigma|y)\,,
\end{align*}
where $p_\sigma \defeq p_X * \cN(0,\sigma^2I)$, and 
\begin{align}\label{eq:cond_noise}
\!\!\!\post(\sigma|y) \propto \frac{\prior(\sigma)}{\sigma^{d}}\,\E_{X\sim p_X}\exp\Bigl(-\frac{1}{2\sigma^2}\,\|X-y\|^2\Bigr)\,.
\end{align}
\end{proposition}
\noindent This allows us to express the optimal blind variance-scaled score function as an integral over the noise posterior $\Pi$, i.e.,
\begin{align}\label{eq:mixture_score}
  \!\! \! r^\star(y) \defeq f^\star(y) - y = \int \sigma^2\, \nabla \log p_\sigma(y)\dd \post(\sigma|y)\,.
\end{align}
\subsection{Derivation of the implicit noise schedule}\label{sec:theory_implicit_noise}
We now examine the dynamics arising from the optimal blind score function. To start, we consider the general continuous-time dynamics of \eqref{eq:empirical_sde} 
with a perfectly trained denoiser, and an arbitrary sequence $(a_t)_{t\geq 0}$, which are given by 
\begin{align}\label{eq:optimal_sde}
\begin{split}
\dd X_t^\star &= r^\star(X_t^\star)\dd t + \sqrt{2a_t} \dd B_t 
\end{split}
\end{align}
where we initialize with  $X_0^\star \sim \cN(0, \sigma_0^2I)$. 

We consider the following ansatz: for all $t$, $X_t^\star$ is approximately distributed as $p_{\sigma_t}$, for some noise sequence $(\sigma_t)_{t\ge 0}$ (to be derived).
Further, let us suppose that when $y = X_t^\star$, then the conditional distribution $\post(\sigma|y)$ in~\eqref{eq:cond_noise} concentrates on the true noise level $\sigma_t$.
As a result, ~\eqref{eq:mixture_score} collapses to a Laplace approximation. 
This suggests considering the following process in which we replace the conditional average over noise levels with $\sigma_t$:
\begin{align}\label{eq:ideal_sde}
\begin{split}
\dd X_t &= r_{\sigma_t}^\star(X_t)\dd t + \sqrt{2a_t} \dd B_t \defeq \sigma_t^2\, \nabla \log p_{\sigma_t}(X_t)\dd t + \sqrt{2a_t} \dd B_t\,.
\end{split}
\end{align}
However, along~\eqref{eq:ideal_sde}, the evolution of the density is given by SDE theory, in particular by the Fokker--Planck equation.
In order to be consistent with the ansatz $\msf{Law}(X_t) \approx p_{\sigma_t}$, this implies the following ODE for $\sigma_t$ (see Appendix~\ref{app:opt_schedule}):
\begin{align}\label{eq:sig_ode_main}
    \frac{1}{2}\,\partial_t(\sigma_t^2) = -\sigma_t^2 + a_t\,.
\end{align}
Solving this ODE, we arrive at the following result.

\begin{proposition}\label{prop:opt_schedule}
The ideal SDE process~\eqref{eq:ideal_sde} satisfies $\msf{Law}(X_t) = p_{\sigma_t}$ for all $t\ge 0$, provided that $X_0\sim p_{\sigma_0}$ and
\begin{align}\label{eq:noise_formula}
    \sigma_t^2 = \sigma^2_0 e^{-2t} + 2\int_0^t a_se^{-2(t-s)}\dd s\,.
\end{align}
\end{proposition}
\noindent We record two properties of this emergent noise evolution.
\begin{lemma}\label{lem:decr_noise}
    If $t\mapsto a_t$ is decreasing and $a_0 \le \sigma_0^2$, then $t\mapsto \sigma_t$ is decreasing.
\end{lemma}

\begin{lemma}\label{lem:noise_to_zero}
    If $t\mapsto a_t$ is decreasing and $a_t \to 0$ as $t\to\infty$, then $\sigma_t \to 0$ as well.
\end{lemma}

Henceforth, we assume that the conditions of Lemma~\ref{lem:noise_to_zero} hold.
In brief, this discussion suggests that if $\post(\cdot|X_t)$ concentrates around $\sigma_t$ such that
\begin{align}\label{eq:noise_approx}
    \int \sigma^2\, \nabla \log p_\sigma(y) \dd \post(\sigma|X_t) \approx \sigma_t^2\, \nabla \log p_{\sigma_t}(X_t)\,,
\end{align}
and if $a_t \to 0$, then heuristically we expect $\msf{Law}(X^\star_t)\approx \msf{Law}(X_t) \to p_X$ as $t \to \infty$.

We stress that at no point during the sampling process will the blind denoiser be given information about the schedule $\sigma_t$. Instead, it must adapt to this sequence automatically by estimating $\sigma_t$ from the current input $Y_t$. In other words, if $\post(\cdot|X_t)$ concentrates, the true $\sigma_t$ for which $p_{\sigma_t} = \pdata * \cN(0,\sigma_t^2 I)$ can be inferred from a single input $Y_t$. This is because $p_{\sigma}(Y_t)$ would be negligible for all $\sigma$ but a small set around $\sigma_t$. As a result, under the concentration assumption, the ideal dynamics in the reverse process should follow the true noise sequence $\sigma_t$.

\subsubsection*{Empirical verification of the concentration assumption}\label{ssec:analytic_denoiser}
Since our proof relies on concentration of $\Pi(\cdot|y)$, we verify this main assumption empirically on synthetic data, before stating the rest of theoretical results. Results on real data are postponed to Section~\ref{sec: empirical results}. 

We study an analytical blind denoiser model when $p_X$ is a $2$-component mixture of Gaussians supported on a $\msf k$-dimensional subspace in $\mathbb{R}^{d}$. 
Here, the optimal score function $\nabla \log p_\sigma$ is known in closed form,  
allowing us to isolate the  error induced by the uncertainty in noise estimation. 
Our analytical blind denoiser will be given by a maximum likelihood estimate (MLE) of $\Pi(\cdot|y)$ for a noisy sample $y$:\looseness-1
\begin{align*}
 \widehat f(y) \defeq y + \widehat{\sigma}^2 \nabla \log p_{\widehat{\sigma}}(y)\,, \quad \widehat{\sigma} = \arg\max \Pi(\cdot|y)
\end{align*}

Figure~\ref{fig:mle-gaussian-mixture} shows the distribution of estimated noise values for a mixture of two Gaussians of dimensionality $\msf k$ in a $d$-dimensional ambient space. Estimates are broadly distributed when $\msf k\approx d$, but concentrate for $\msf k \ll d$, illustrating a ``blessing of dimensionality''. Figure~\ref{fig:samples_mixture_2Gaussians} shows samples generated via Algorithm~\ref{alg:inf} using our analytical denoiser.
In these experiments, $h = 0.3$ and the dynamics are deterministic ($a_t = 0$). Similar results are obtained when $a_t > 0$. For more details, see Appendix~\ref{app:gaussian}. 

\begin{figure}[t]
    \centering
   \includegraphics[width=.48\linewidth]{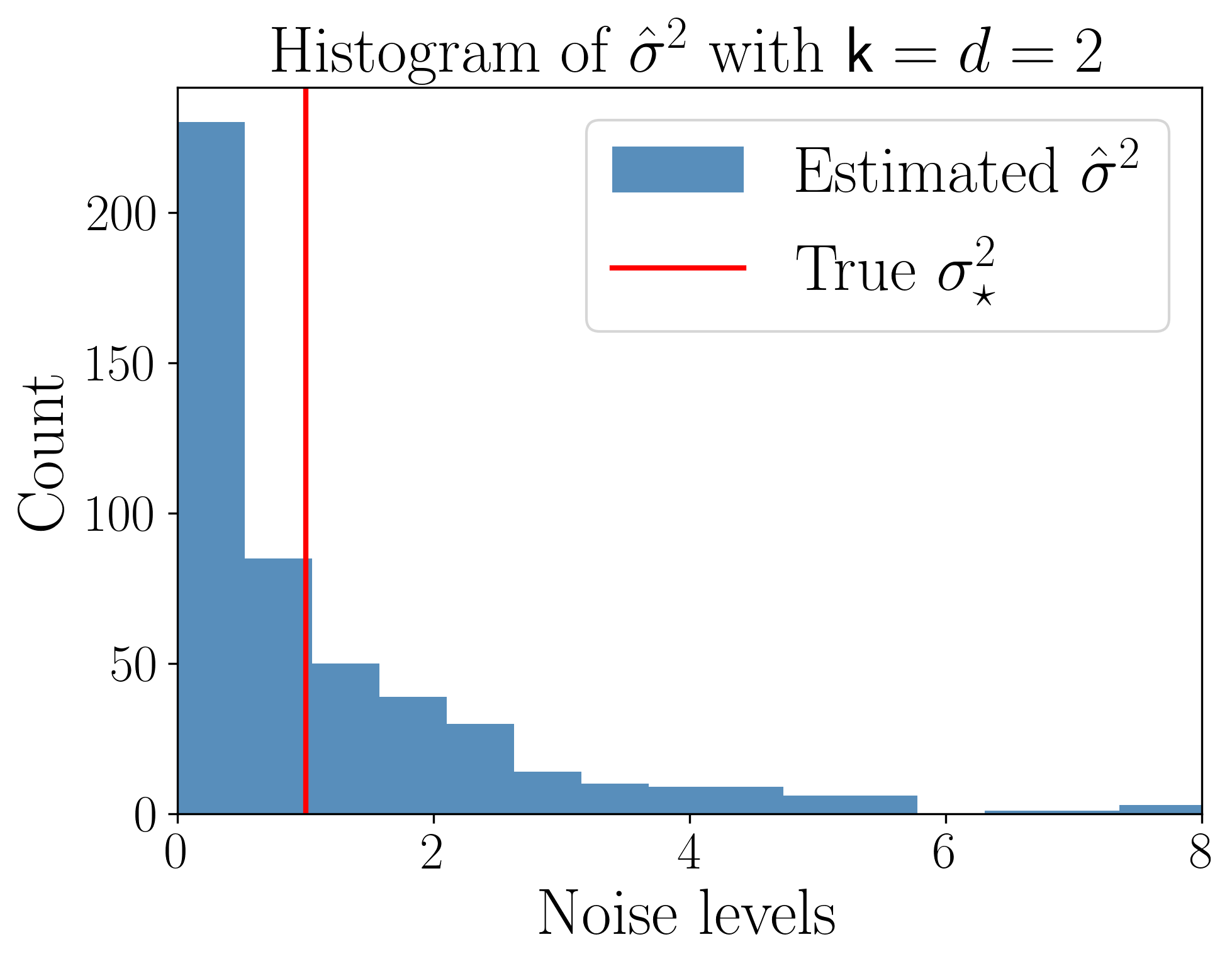}     
   \hfil
   \includegraphics[width=.48\linewidth]{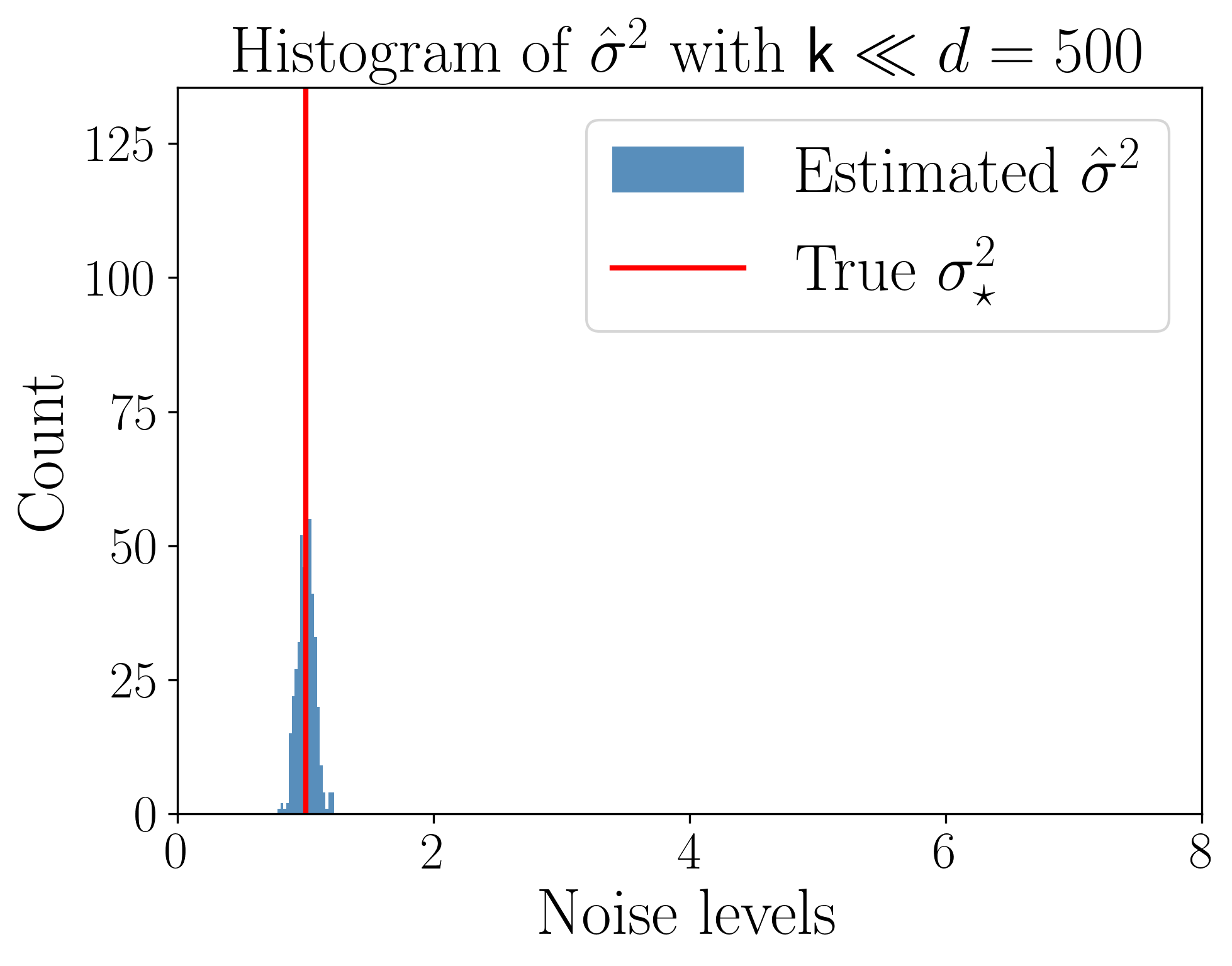}     
   \caption{Empirical {density} of maximum likelihood estimates $\widehat{\sigma} = \arg\max \Pi(\sigma|y)$ for a mixture of two Gaussians with intrinsic dimensionality $\msf k$ = $2$ in an ambient space of dimensionality $d$. Estimates are broadly distributed for small $\msf k = d$ \textbf{(left)}, but concentrated for $\msf k^2 \ll d$ \textbf{(right)}.
    }
    \label{fig:mle-gaussian-mixture}
\end{figure}
\begin{figure}[t]
    \centering
   \includegraphics[width=.49\linewidth]{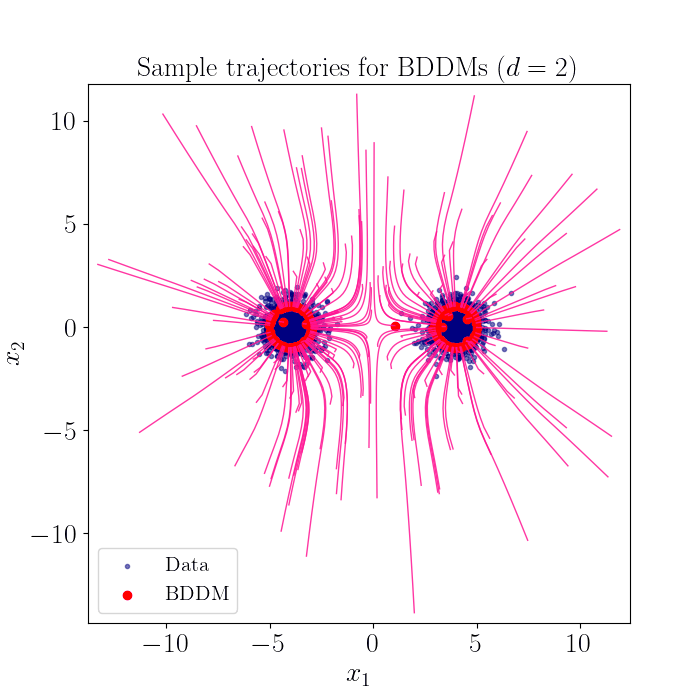}  
   \hfil
   \includegraphics[width=.49\linewidth]{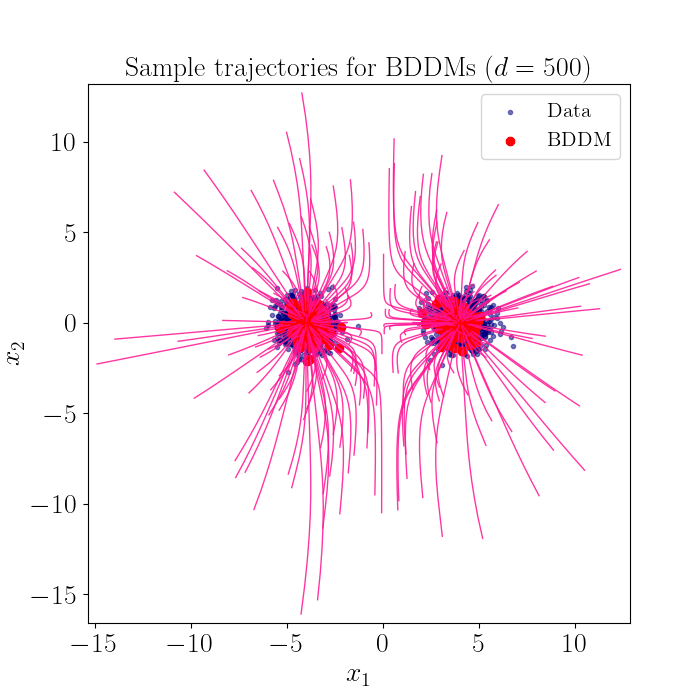}     
    \caption{\textbf{Illustration of Corollary~\ref{cor:sampling_cor}:} Example trajectories and samples for an analytical blind denoiser applied to a mixture of two Gaussians with $\msf k = 2$. For $d=\msf k=2$ \textbf{(left)}, sampling fails, due to errors in the MLE estimates of noise level. For $d=500 \gg \msf k^2$ \textbf{(right)}, sampling is successful, illustrating the blessings of dimensionality.  
}\label{fig:samples_mixture_2Gaussians}
\end{figure}

\subsection{Error bound without discretization}\label{sec:main_result}
In this section, we study the conditions required for our concentration assumption to hold, and then quantify the excess error arising from deviations from the perfect concentration. 

Suppose that we have an empirical minimizer $\widehat f_\theta$ that will act as the blind denoiser, and let $\widehat p_t \defeq \msf{Law}(Y_t)$ along the trained dynamics~\eqref{eq:empirical_sde}. For now, we omit discretization error, deferring this discussion to Section~\ref{sec:theory_discretization}. Our goal is to bound the KL divergence between $\widehat p_T$, the distribution corresponding to our algorithm at time $T$, and the {target distribution with some small additional noise $p_{\sigma_T} = p_X * \cN(0,\sigma_T^2I)$}. 

By a standard application of Girsanov's theorem, we can decompose this error into the following terms:
\begin{align}\label{eq:girsanov_33}
\begin{split}
\KL(p_{\sigma_T} \| \widehat p_T)
    &\lesssim \KL(p_{\sigma_0}\|\widehat p_0) 
     +\int_0^T  \frac{\|\widehat f_\theta - f^\star\|^2_{L^2(p_{\sigma_t})}}{a_t} \dd t +\int_0^T\frac{\|r^\star_{\sigma_t} - r^\star\|^2_{L^2(p_{\sigma_t})}}{a_t}\dd t\,.
\end{split}\raisetag{1.5\baselineskip}
\end{align}
The first term on the right-hand side corresponds to the initialization error. The second term measures how close the trained denoiser is to the optimal one (the ``score error''). Finally, the third term is related to the error in the approximation~\eqref{eq:noise_approx}, i.e., the error in estimating the noise level. 
This last term is the primary novelty over prior work on diffusion models, so we focus our attention on it.

Before stating our main results, we require the following assumptions and definitions. The support of our data distribution will be denoted by $\cX \defeq \supp(p_X)$, and we let $\msf k$ denote  the ``{intrinsic dimension}'' of $p_X$, defined as follows.

\begin{definition}[Intrinsic dimension]\label{def:intdim}
If $\pdata$ has support $\cX$, we define the \emph{intrinsic dimension} of $\pdata$ at scale $r_0$ to be $\msf k \defeq 1 + \log N_{\cX}(r_0)$, where $N_{\cX}(r_0)$ denotes the minimal number of balls of radius $r_0$ needed to cover $\cX$.
\end{definition}

For example, if $\cX$ is a $k$-dimensional subspace and $\cX$ has diameter $R$, then $\msf k \asymp k\log(R/r_0)$. But the {intrinsic dimension captures a broader range of situations} such as manifold structure, or when $\cX$ is actually full-dimensional but resembles a $k$-dimensional manifold at scale $r_0$.
Our proofs require $r_0 = \sigma_T^2/(\sigma_0\sqrt d)$ as the choice of scale.\looseness-1

Our assumptions are the following.
\begin{description}
\item[\textbf{(A1)}] $p_X$ has finite second moment, denoted $\mathsf m_2^2$.
\item[\textbf{(A2)}] $\pdata$ has low intrinsic dimensionality, with $d \ge C\,(\msf k + \log(\sigma_0/\sigma_T))$ for a sufficiently large absolute constant $C > 0$.
\item[\textbf{(A3)}] For $\epsilon_{\rm BD} > 0$,  $\int_0^Ta_t^{-1}\|\widehat f_\theta - f^\star\|^2_{L^2(p_{\sigma_t})}\dd t \leq \epsBD^2$.
\end{description}

Note that \textbf{(A1)}, \textbf{(A2)}, and \textbf{(A3)} are relatively mild. In particular, \textbf{(A2)} is inspired from a slew of previous empirical observations \citep{olshausen1996natural, roweis2000nonlinear, chandler2007estimates, Hnaff2014TheLL,pope2021intrinsic, Brown+23Manifolds} and has been leveraged in several other works on sampling via diffusion models \citep{li2024adapting, pooladian2024plug, liang2025low}. \textbf{(A3)} is analogous to the standard assumption that the score functions are accurately learned in $L^2$.
In its current form, it is somewhat obscure; we provide further discussion in Section~\ref{sec:prior} when we specialize to specific noise schedules.

We are now in a position to state our first main result about excess error due to noise estimation.

\begin{theorem}\label{thm:thm_main}
Under \textbf{(A1)}--\textbf{(A3)}, the KL divergence between the (reverse) processes \eqref{eq:empirical_sde} and \eqref{eq:ideal_sde} is bounded by  
\begin{align*}
    {\rm KL}(p_{\sigma_T}\|\widehat p_T) \lesssim \frac{\msf m^2_2}{\sigma_0^2} + \epsilon_{{\rm BD}}^2 + \bigl(\frac{\msf k^3}{d} + \frac{\msf k^5}{d^2}\bigr) \int_0^T \frac{\sigma_t^2}{a_t}\dd t\,.
\end{align*}
\end{theorem}
\vspace{-2mm}
\noindent The first term (initialization error) is made small with a suitably large choice of $\sigma_0^2$, and the second term (the ``score error'') is small if the training is successful.
The third term, which scales with $\msf k$, is discussed next.

\subsection{Interpretation as a Bayesian problem}\label{sec:theory_bayesian}
We now provide an overview of our proof with rigorous details deferred to the appendix.
As the first two terms in Theorem~\ref{thm:thm_main} are standard, we focus on the novel term $\int_0^T a_t^{-1}\|r^\star_{\sigma_t}-r^\star\|_{L^2(p_{\sigma_t})}^2\dd t$ which captures the error in estimating the noise level. 

Note that for $X_t \sim p_{\sigma_t}$ we can re-write the integrand with respect to deviations of noise posterior, $\Pi(\sigma|X_t)$, from the perfect posterior concentrated on a Dirac mass $\delta_{\sigma_t}$
\begin{align*}
    \|(r_{\sigma_t}^\star - r^\star)(X_t)\|^2 &= \Bigl\lVert \int \sigma^2\, \nabla \log p_{\sigma}(X_t) \dd(\delta_{\sigma_t} - \post(\sigma |X_t))\Bigr\rVert^2\\
    &= \Bigl\lVert\iint_{\sigma}^{\sigma_t} \partial_{\omega} (\omega^2\, \nabla \log p_{\omega}(X_t)) \dd \omega \dd \post(\sigma|X_t)\Bigr\rVert^2\,,
\end{align*}
where we used the fundamental theorem of calculus in the last line. A simple calculation shows that
\begin{align*}
\begin{split}
\partial_\omega(\omega^2\, \nabla \log p_\omega(y))\!=\!\omega^{-3}\,{\rm Cov}(X,\|X-Y\|^2\mid Y=y)\,,
\end{split}
\end{align*}
under the distribution of $X \sim \pdata$, $Y \sim \cN(X,\omega^2 I)$.
We show in Corollary~\ref{cor:helper_cor} that the conditional covariance is bounded in terms of the intrinsic dimension $\msf k$, essentially leading to a bound on the above term of order
\begin{align*}
  \sigma_t^3\, \msf k^3 \int |{\sigma_t^{-2}} - {\sigma^{-2}}|^2\dd\post(\sigma|X_t)\,.
\end{align*}
The remainder of the analysis is a frequentist analysis of a Bayesian method: We have an observation $X_t \sim p_{\sigma_t}$, where $\sigma_t$ is the ``ground truth'' noise level. We ask whether the posterior distribution on the noise level $\post(\cdot|X_t)$ concentrates on $\sigma_t$ given a single sample. 
For interpretability, we focus on the class of power law priors $\prior(\sigma)\propto \sigma^{\alpha-3}$ on $[\sigma_T,\sigma_0]$ for $\alpha\in\R$, $\alpha = O(1)$, and we use the following change of variables.

\begin{lemma}\label{lem:log-derivatives}
If $\sigma \sim \post(\cdot|y)$ in~\eqref{eq:cond_noise}, then $\lambda \defeq \sigma^{-2}$ is distributed according to
\begin{align}\label{eq:cond_lambda}
    \ppost(\lambda|y)\propto \lambda^{(d-\alpha)/2}\,\E_{X\sim p_X}\bigl[\exp\bigl(-\tfrac{\lambda}{2}\,\|X-y\|^2\bigr)\bigr]\,.
\end{align}
Writing $\ell(\lambda|y) \defeq -\log\ppost(\lambda|y)$, then
\begin{align*}
    &\ell'(\lambda|y) = -\frac{d-\alpha}{2\lambda} +\frac{1}{2}\,\E_{q_\lambda}[\|X-Y\|^2\mid Y=y]\,, \\[0.25em]
    &\ell''(\lambda|y) = \frac{d-\alpha}{2\lambda^2} - \frac{1}{4}\,{\rm Var}_{q_\lambda}(\|X-Y\|^2\mid Y=y)\,,
\end{align*}
where, in an abuse of notation, we use $q_\lambda$ to denote the joint distribution for which $X \sim \pdata$, $Y \sim \cN(X, \lambda^{-1} I)$.
\end{lemma}
Notice that there is no reason for $\ell''(\cdot|y)$ to always be convex, given the presence of a negative sign on the second term. However, our \emph{low intrinsic dimensionality assumption} \textbf{(A2)} nevertheless allows us to show that  the {noise variance posterior concentrates on the ground truth {signal}} $\lambda_t \defeq \sigma_t^{-2}$. 

\begin{proposition}\label{prop:nut_informal}
Under \textbf{(A1)}--\textbf{(A2)}, with high probability over $X_t \sim p_{\sigma_t}$, and for $\lambda \sim \ppost(\cdot|X_t)$, 
\begin{align*}
\E|\lambda - \lambda_t|^2
\lesssim \lambda_t^2\,(d^{-1} + \msf k^2 d^{-2})\,.
\end{align*}
\end{proposition}
\noindent {See~\eqref{eq:noise_formal} for a precise result.} Combining these ingredients yields the claimed bound on the noise level term.

\subsection{Discretization, noise, and step size schedules}\label{sec:theory_discretization}
\paragraph{A family of convenient {diffusion} schedules.}
Although our results apply to general choices of $(a_t)_{t\ge 0}$, for the sake of exposition we specialize to a particular convenient family in order to state our next results in a more interpretable form.
Namely, if we take $a_t = \ta \sigma_t^2$ for some $\ta \in (0,1)$, then the ODE~\eqref{eq:sig_ode_main} becomes particularly easy to solve: $\sigma_t = \sigma_0 e^{-(1-\ta) t}$. In fact, this is the choice made in \eqref{eq:bddm_sampler_intro} in which $\ta = \beta$. We henceforth fix this particular choice. 

Figure~\ref{fig:learning_gaussian_denoising} shows samples and evolution of noise level for a neural network blind denoisers trained on Gaussian data. Under an appropriate choice $a_t = \frac12 \sigma_t^2$ and $h=0.5$, we see that $\hat \sigma_k \simeq \sigma_t$ as predicted by Proposition~\ref{prop:opt_schedule} when having low intrinsic dimensionality. See Appendix~\ref{app:gaussian_gaussian_mixture} for details. 

\begin{figure}[t]
    \centering
   \includegraphics[width=.44\linewidth]{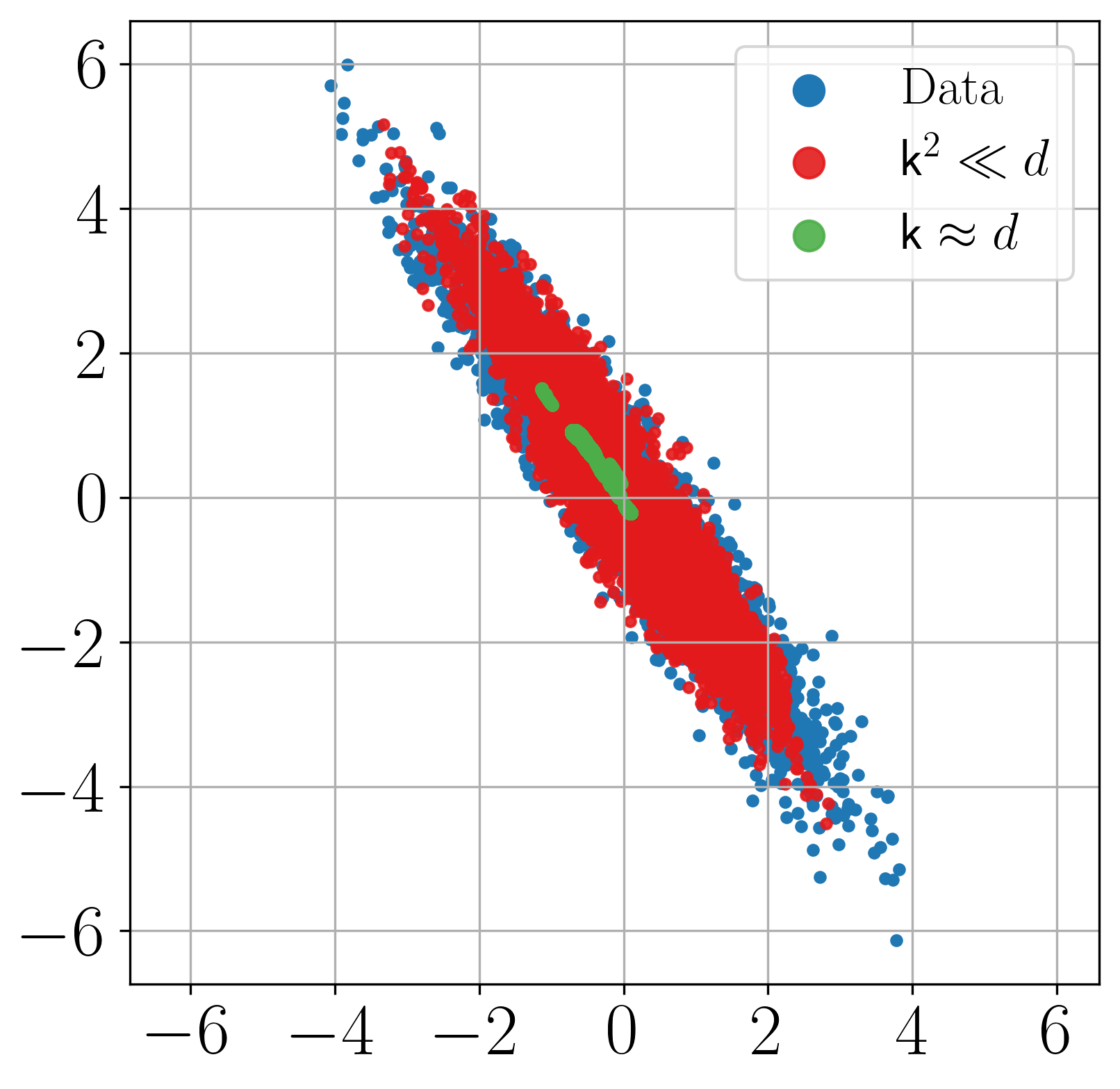}
\includegraphics[width=.55\linewidth]{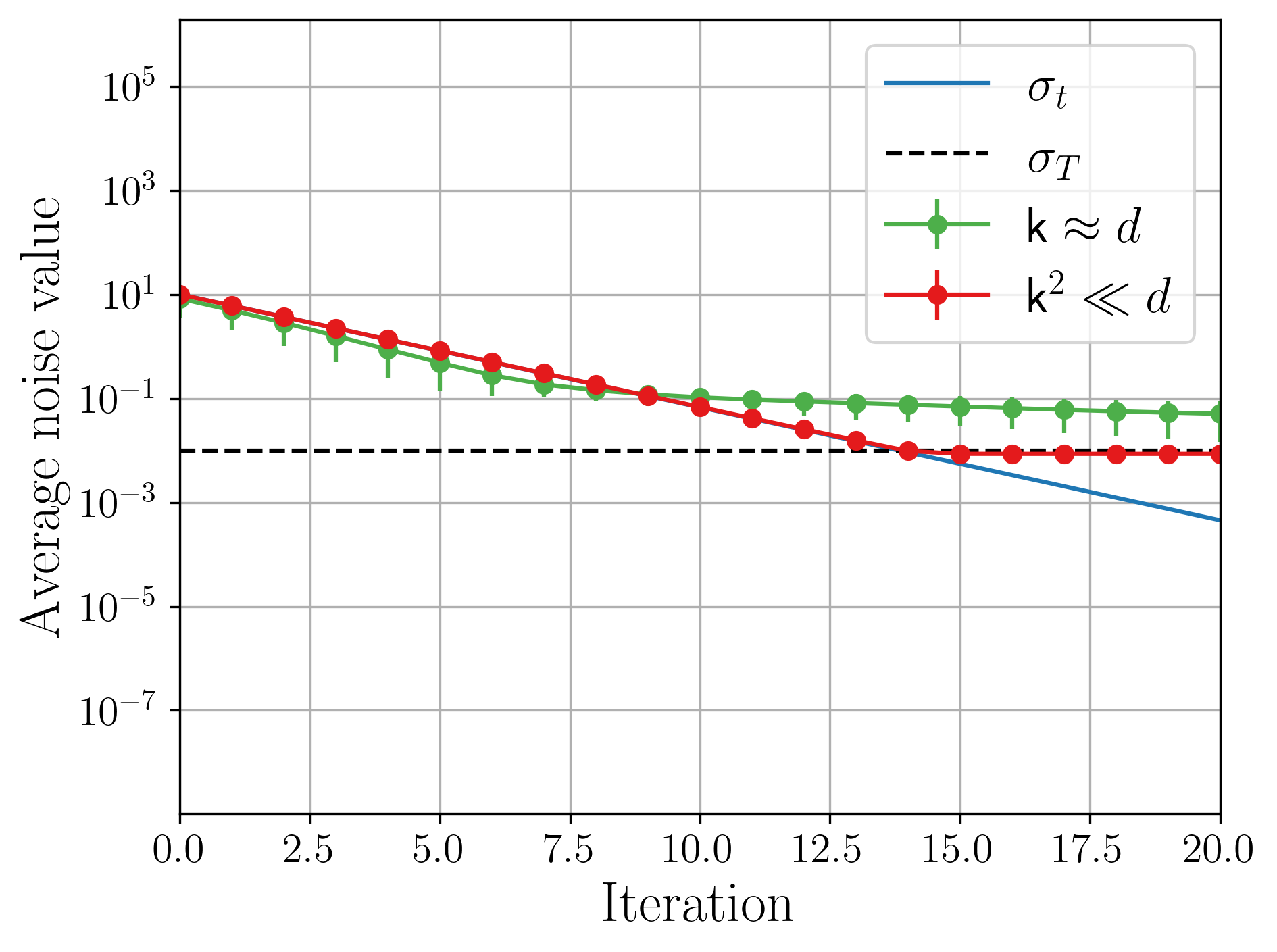}
    \caption{Sampling performance of BDDMs trained on Gaussian data with intrinsic dimension $\msf k =2$ and two different input dimensions $d \in \{2,100\}$.  Generated samples \textbf{(left)} and evolution of the estimated noise level, corresponding to  Proposition~\ref{prop:opt_schedule} \textbf{(right)}.}
\label{fig:learning_gaussian_denoising}
\end{figure}
\paragraph{Discretization error.}
If the SDE \eqref{eq:empirical_sde} is discretized, then in addition to the error terms present in Theorem~\ref{thm:thm_main}, we will also incur a discretization error.
In order to establish an iteration bound which also scales with the intrinsic dimension, we consider the exponential Euler integrator (see Appendix~\ref{app:exp_euler}).
With this choice, the discretization error takes the following form:
Writing $t_- \defeq \lfloor t/h \rfloor h$ where $h$ is the step size, thus $N = T/h$ is the total number of iterations,
\begin{align*}
    \textstyle{\msf{Disc}}(h) = \int_0^T a_t^{-1}\,\E\|f_{\sigma_{t_-}}^\star(X_{t_-}) - f_{\sigma_{t}}^\star(X_t)\|^2 \dd t\,.
\end{align*}
Here, $f^\star_\sigma$ denotes the Bayes denoiser when the noise level is known: $f^\star_\sigma(y) \defeq \E_{q_\sigma}[X\mid Y=y]$.\footnote{{Indeed ${\msf{Disc}}(h)$ arises from \eqref{eq:girsanov_33} if we also incorporate discretization error.}} We analyze the discretization error and establish the following bound.

\begin{theorem}\label{thm:disc_error}
    Under \textbf{(A1)}--\textbf{(A3)} and with the exponential Euler scheme, for $h\lesssim 1$,
    \begin{align*}
        \msf{Disc}(h)
        &\lesssim \bigl(C_{\ta, 1}\, \msf k^3 h^2 + C_{\ta, 2}\, \msf k h\bigr)\log\bigl( \frac{\sigma_0}{\sigma_T}\bigr)\,,
    \end{align*}
    where $C_{\ta,1} \defeq \frac{(\ta-1/2)^2}{\ta\,(1-\ta)}$ and $C_{\ta,2} \defeq \frac{1}{1-\ta}$.
\end{theorem}
\noindent Remarkably, the first error term vanishes \emph{exactly} when $\ta = 1/2$.
Therefore, {our analysis reveals a canonical choice of $a_t$ sequence, namely $a_t = \frac{1}{2}\sigma_0^2 e^{-t}$, which leads to smaller discretization error}.

Moreover, if $\ta = 1/2$, the error bound shows that {we can provably choose a constant step size $h$, scaling with the intrinsic dimension $\msf k$}.
This is in {stark contrast to theoretical guarantees for the variance-preserving DDPM algorithm, which requires carefully tuned exponentially decaying step sizes to compensate for the singularity of the score {for time $t \to T$}} \citep{benton2023nearly,conforti2025kl}.

\subsection{End-to-end sampling guarantees}\label{sec:perceptual}

At this point, we nearly have a complete {sampling} guarantee for BDDMs.
However, Theorem~\ref{thm:thm_main} only guarantees closeness to {$p_{\sigma_T} = \pdata * \cN(0,\sigma_T^2I)$}, not to $\pdata$.
One approach is to simply assume that {sampling from $p_{\sigma_T}$ is our goal all along for some small $\sigma_T$,} which is
sometimes adopted in the literature \citep{gatmiry2026high}.\looseness-1

More commonly, $p_{\sigma_T}$ is used as a technical device to sample from $\pdata$, which is known as \emph{early stopping}.
How small do we need to take $\sigma_T$ in order to ensure that $p_{\sigma_T}$ and $\pdata$ are close?
This is important, since both the noise estimation error and the discretization error depend on $\sigma_T$ (either explicitly or implicitly through the value of $T$). However, if we make the judicious choice $a_t = \tfrac12 \sigma_t^2$ which naturally arises from the discussion surrounding Theorem~\ref{thm:disc_error}, our analysis becomes considerably simpler, {and we can bound the distance between $p_X$ and $ p_{\sigma_T}$ through choices of $\sigma_T, \sigma_0$ and $h$.}
To state our result, we use the bounded Lipschitz metric, which metrizes weak convergence:
\begin{align*}
    {\msf D}_{{\rm BL}}(p,q) \defeq \sup\{\E_p[f] - \E_q[f] : f \in {\rm BLip}\}\,,
\end{align*}
where ${\rm BLip}$ is the space of $1$-bounded, Lipschitz functions.\footnote{Functions which are $1$-Lipschitz and with $-1 \leq f \leq 1$.}

\begin{corollary}\label{cor:sampling_cor}
Assume \textbf{(A1)}--\textbf{(A3)} and let $p_{\msf{alg}}$ be the output of BDDM with exponential Euler discretization with $a_t = \tfrac12\, \sigma_t^2$. For $\varepsilon \in (0,1)$, choose $\sigma_0 \asymp \msf m_2/\eps$, $\sigma_{T}\asymp  \epsilon/\sqrt{d}$, and $h\asymp \varepsilon^2/(\msf k\log(\msf m_2\sqrt d/\varepsilon^2))$.
Then, ${\msf D}_{\rm{BL}}(p_X, p_{\msf{alg}}) \lesssim \tepsBD + \epsilon$, provided that 
\begin{align*}
   d \gtrsim \frac{\msf k^3}{\varepsilon^2} \log(\msf m_2\sqrt d/\varepsilon^2) \,\,\text{and}\,\, N\asymp \frac{\msf k}{\epsilon^{2}} \log^2(\msf m_2\sqrt d/\varepsilon^2)\,.
\end{align*}
\end{corollary}
\noindent In this result, the definition of the score error has to be slightly modified to $\tepsBD^2 \defeq \int_0^T a_t^{-1}\|\widehat f_\theta - f^\star\|_{L^2(p_{\sigma_{t_-}})}^2\dd t$, since the error only matters at discretization time steps. Figure~\ref{fig:samples_mixture_2Gaussians} illustrates convergence in Corollary~\ref{cor:sampling_cor} for a $p_X$ constructed as a mixture of two Gaussians.

\subsection{On the choice of noise prior}\label{sec:prior}

Finally, we revisit the assumption \textbf{(A3)} on the ``score estimation error''.
In standard DDPM theory, the bound on the score estimation error in $L^2$ which is needed for the discretization analysis can also be written as the \emph{excess risk} of the population score matching loss.
This leads to a remarkable harmony between statistical theory which can identify when the $L^2$ error is small, and the discretization analysis.
In the context of the convenient noise sequence defined above, we prove the analogous result for BDDMs.

\begin{theorem}\label{thm:score_error}
    Suppose that $[\sigma_T,\sigma_0] \subseteq \supp \prior$.
    It holds that
    \begin{align*}
        \epsBD^2
        &\le \bigl(\min_{\sigma_T \le \sigma \le \sigma_0} \ta\,(1-\ta)\,\sigma^3\,\prior(\sigma)\bigr)^{-1}\,\cE(\widehat f_\theta)\,,
    \end{align*}
    where $\cE(\cdot)$ denotes the population excess risk for~\eqref{eq:empirical_loss}.
\end{theorem}
\noindent This result shows that if the neural network learns a good blind denoiser, in the sense of attaining a small excess risk, then the $\epsBD$ quantity in our error analysis is controlled.
Moreover, {our analysis singles out a particularly good choice of noise prior}. 
Namely, when $\prior(\sigma) \propto \sigma^{-3}$ over $[\sigma_T,\sigma_0]$, then $\epsBD^2$ and $\cE(\widehat f_\theta)$ are in fact equal up to a constant, so that the training objective and the error propagation along the dynamics are well-aligned.\footnote{Note that $\prior(\sigma)\!\propto\!\sigma^{-3}$ corresponds to a uniform prior over $\lambda$.\looseness-1} 
Although we state this for a particular family of noise schedules for concreteness, the same approach indeed identifies a well-aligned prior $\prior$ for each schedule $(a_t)_{t\ge 0}$.

\begin{algorithm}[t]
  \caption{Training a blind denoiser}
  \label{alg:train}
  \begin{algorithmic}
    \STATE {\bfseries Input:} Distribution $\prior$, neural network $f_\theta$
    \WHILE{{\texttt{not} \texttt{converged}}}
    \STATE Draw $x_1,\ldots, x_B \sim p_X$
    \STATE Draw $\sigma_1,\ldots,\sigma_B \sim \prior$
    \STATE Draw $z_1,\ldots,z_B \sim \cN(0,I)$
    \STATE Compute $\cL(\theta) = B^{-1}\textstyle\sum_{i=1}^B\|x_i - f_\theta(x_i + \sigma_i z_i)\|^2$
    \STATE Update $\theta \gets \theta - \eta \nabla_\theta \cL(\theta)$.
    \ENDWHILE
  \end{algorithmic}
\end{algorithm}
\begin{algorithm}[t]
  \caption{Sampling using a blind denoising model}
  \label{alg:inf}
  \begin{algorithmic}
    \STATE {\bfseries Input:} Trained neural network $\widehat f_\theta$, stepsize $h > 0$, diffusion coefficients $(a_t)_{t \in [0,T]}$, and $\sigma_{\max},\sigma_{\min} > 0$
    \STATE {\bfseries Initialize} $X_0 \sim \cN(\widehat{m}_X, \sigma_{\max}^2I)$, $k=0$
    \STATE \texttt{keepgoing} $\gets$ \texttt{True}
    \WHILE{{\texttt{keepgoing} == \texttt{True}}}
    \STATE Compute $r_{k} \gets \widehat f_\theta(X_{kh}) - X_{kh}$
    \STATE Compute $\hat \sigma^2_{k} \gets \|r_k\|^2/d$
    \IF{$\hat \sigma_k \leq \sigma_{\min}$}
    \STATE \texttt{keepgoing} $\gets$ \texttt{False}
    \ELSE
    \STATE Draw $\bar\xi \sim \cN(0, (\int_{kh}^{(k+1)h} 2a_t\dd t) I)$
    \STATE Update $X_{(k+1)h}\gets X_{kh} + h\,r_k + \bar\xi$
    \ENDIF
    \STATE {Update $k \leftarrow k+1$}
    \ENDWHILE
  \end{algorithmic}
\end{algorithm}

\section{Empirical results on photographic images}
\label{sec: empirical results}

Experiments on synthetic data shown in previous sections verified that  (1) the noise variance can be accurately estimated from a single noisy observation in high dimensions; (2) the reverse process of Algorithm~\ref{alg:inf} closely adheres to the theoretical implicit schedule of Proposition~\ref{prop:opt_schedule}. 
Do these results extend beyond the synthetic data to real-world models and signals? 
Additionally, can BDDMs offer a substantial gain in sampling performance compared to non-blind models, by eliminating the mismatch between true noise level and a proposed noise schedule? 
To investigate these questions, we consider deep neural networks trained on natural images. We trained a non-blind and blind model to approximate the optimal denoisers in \eqref{eq:tweedie} and \eqref{eq:mixture_score}, respectively. Importantly, the two models are exactly the same in all respects except for noise conditioning.

\textbf{Datasets.} Under the manifold hypothesis, natural images concentrate near a union of low-dimensional manifolds \citep{Brown+23Manifolds}. This implies that natural images have low intrinsic dimensionality, making them a suitable testbed for our theoretical results. We use two popular image datasets, CelebA \citep{liu2015faceattributes} and LSUN (bedroom class) \citep{yu2015lsun}. These datasets are complex enough to capture real-world structure while remaining simple enough to train smaller-size  networks without text conditioning. All images are downsampled to $80 \times 80$ resolution.\looseness-1

\textbf{Architecture and training.} UNet \citep{ronneberger2015u} is the most popular denoising and score estimation architecture. For both blind and non-blind models, we use the simplest UNet architecture with $13$ million parameters (small compared to many architectures which use hundreds of millions of parameters), and train all models from scratch. 

For the non-blind models, we choose the standard noise level embedding. 

Blind models are trained to minimize the empirical loss~\eqref{eq:empirical_loss}, while non-blind models are trained for the noise-conditioned counterpart. For more experimental details, see Appendix~\ref{app:images}.  

{Can the noise level be accurately estimated from a single noisy image by a trained neural network denoiser?}
If natural images are truly low dimensional, then we expect $\Pi(\cdot|x_{\sigma}) \simeq \delta_\sigma$. Additionally, the inductive biases of the neural network should allow for leveraging the concentration to get a precise estimate of true variance. If both of these conditions are satisfied, then the performance gap between blind and non-blind denoisers will be small. Figure~\ref{fig:psnrs} compares denoising error of the two models on two datasets in terms of their peak signal-to-noise ratio (PSNR) averaged across the data, 
$    {\rm PSNR}(x,\hat x) \defeq - 10 \log_{10}\| x -  \hat{x}\|^2\, $
where $\hat x$ is the output of a blind or non-blind denoiser.
Performance results in Figure~\ref{fig:psnrs} demonstrate that {blind denoisers are as good as the non-blind ones, implying that the blind model estimates the noise level from data almost perfectly}. 
\begin{figure}
    \centering
   \includegraphics[width=.3\linewidth]{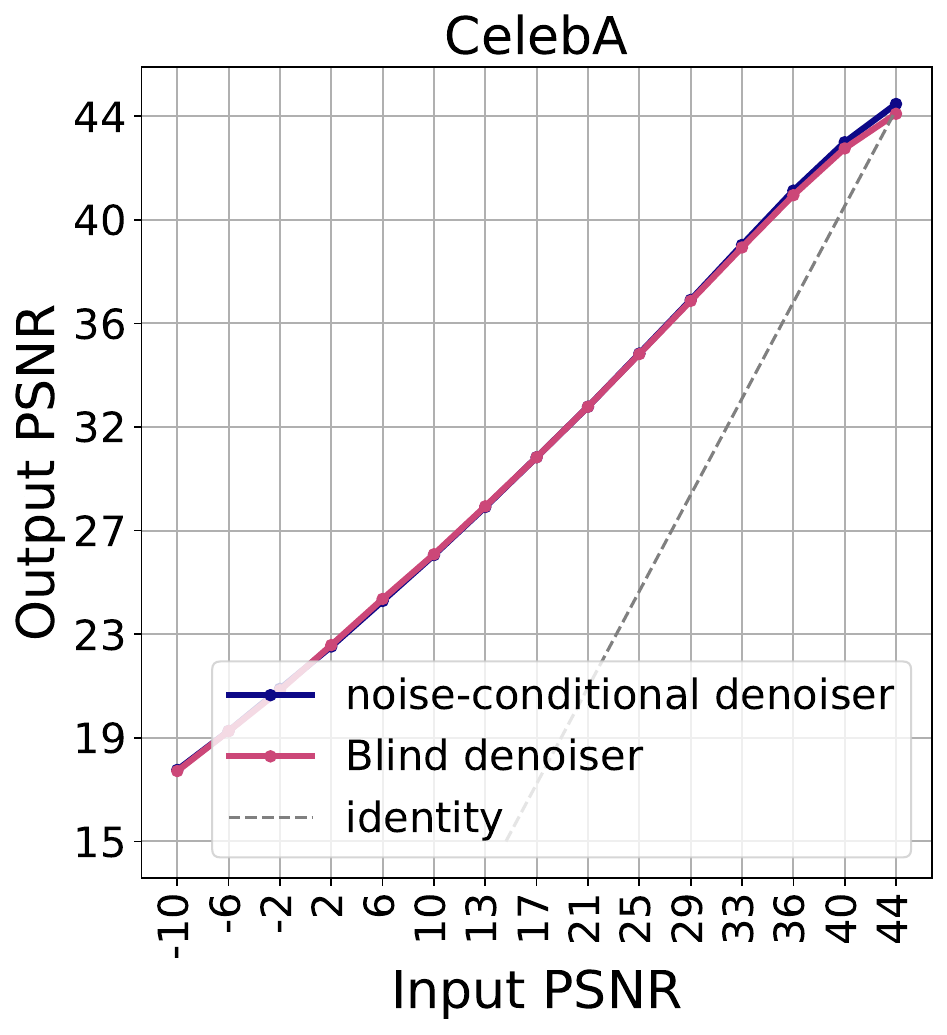}     
   \hfil
    \includegraphics[width=.3\linewidth]{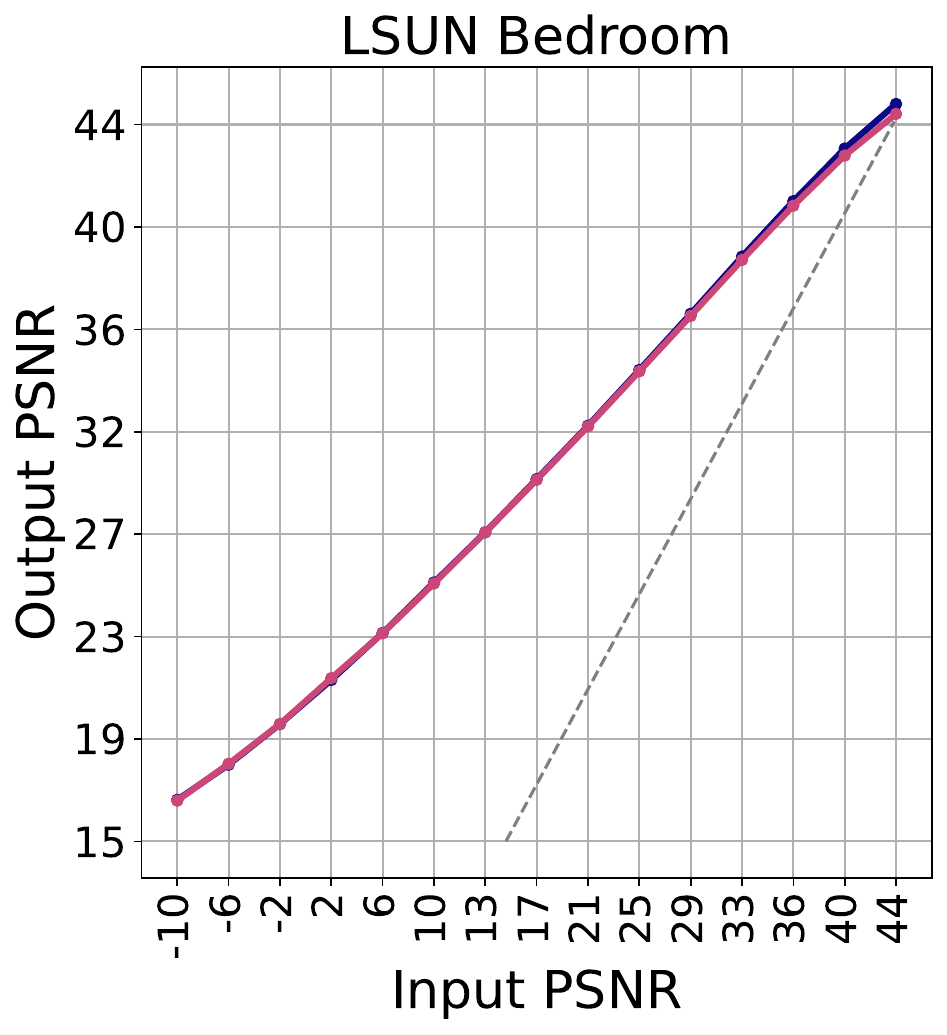}    
    \includegraphics[width=.65\linewidth]{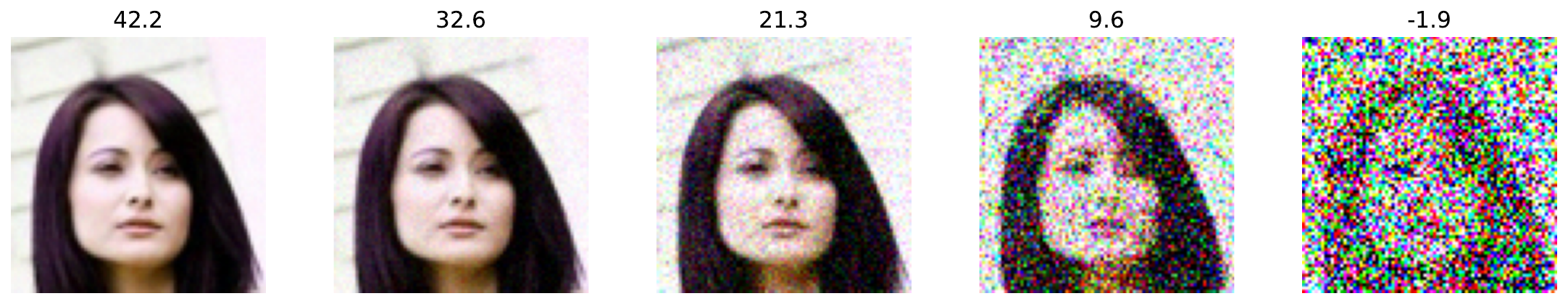}    
    \includegraphics[width=.65\linewidth]{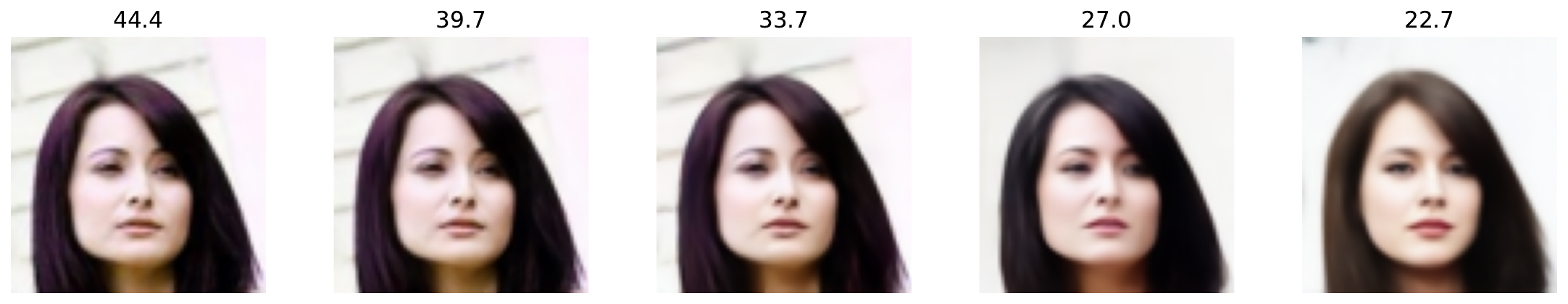}    
    \includegraphics[width=.65\linewidth]{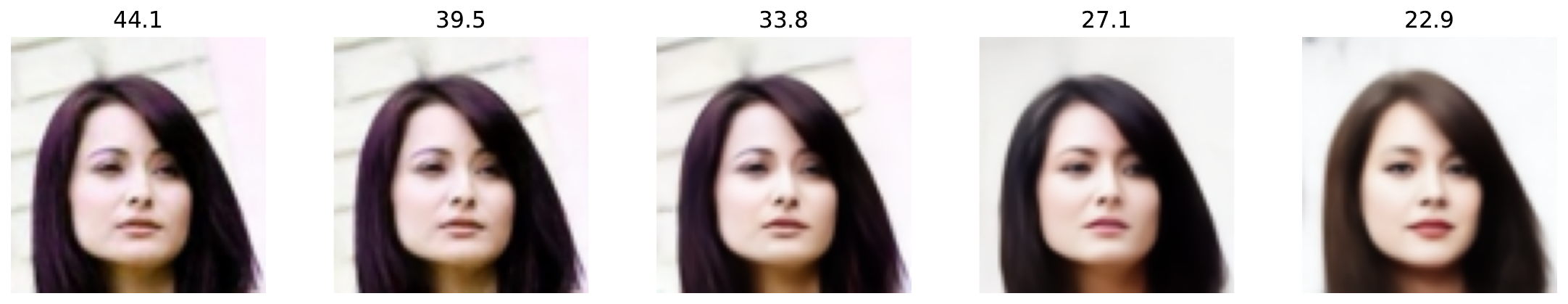}        
    \caption{\textbf{Top panel.} Comparison of denoising performance of blind  and non-blind denoisers. Performance is evaluated on the test sets of CelebA and Bedroom class of LSUN datasets, and is reported in terms of PSNR. For consistency, the noise level on the horizontal axis is also expressed as ${\rm PSNR}(x,x_\sigma)$. 
        \textbf{Bottom panel.} An example test image with corresponding PSNR values. \textbf{Top:} Noisy images.  \textbf{Middle:} Denoised image by a non-blind denoiser. \textbf{Bottom:} Denoised image by a blind denoiser. 
        }    
    \label{fig:psnrs}
\end{figure}

{Now we sample from the space of natural images via Algorithm~\ref{alg:inf}.} 
Hyperparameters are set to $\sigma_{\max} = 4, \sigma_{\min} = 0.05, h = 0.05, a_t = 0.3 \hat{\sigma}_t^2$, which leads to a total number of steps $N \approx 1000$. 
We compare to the non-blind denoisers via variance exploding (VE) of DDPM algorithm \citep{SonMenErm21DDIM}, with a $\log \sigma$ schedule with  $N = 1000$.

Figure \ref{fig:samples_celeba_1k} shows the output of both sampling algorithms (see Figures~\ref{fig:samples_celeba_17k} and \ref{fig:samples_celeba_100} for additional examples and Figure~\ref{fig:samples_bed} for LSUN samples). Surprisingly, {samples generated by BDDMs appear to have higher visual quality than samples generated by the non-blind DDMs}. This is while the two models have the same denoising performance as shown in Figure~\ref{fig:psnrs}. Nevertheless BDDM samples appear to be drawn from a distribution more similar to the underlying true density than DDPM samples.
\emph{This suggests that the differences are due to the sampling algorithms.}

\begin{figure}[]
    \centering
   \includegraphics[width=.8\linewidth]{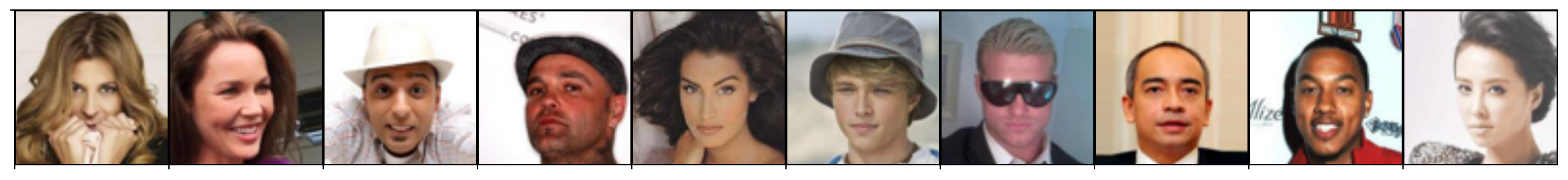}  
   \includegraphics[width=.8\linewidth]{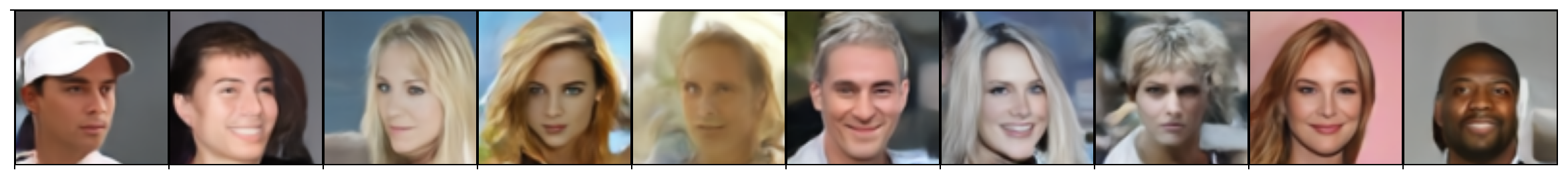}          
   \includegraphics[width=.8\linewidth]{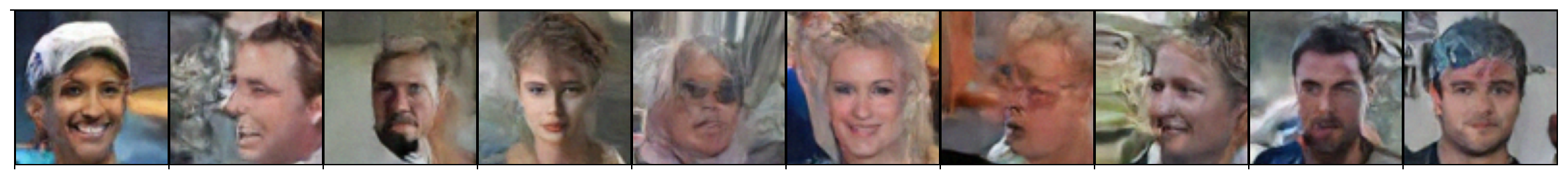}
    \caption{Comparison of samples from BDDM and VE-DDM.
\textbf{Top row:} Randomly selected subset of training images from the CelebA dataset. \textbf{Second row:} Samples generated by BDDM with $N \approx1000$. \textbf{Third row:} Samples generated by a non-blind DDM (VE-DDPM), with $N=1000$. 
    Samples in each column are initialized with the same random seed, and use matched injected noise. Seeds are random and {\em not curated for quality}. See Appendix \ref{app:more_samples} for lower and higher $N$, and for LSUN dataset. 
    }
    \label{fig:samples_celeba_1k}
\end{figure}
We hypothesize that using non-blind model in a reverse process with an explicit schedule incurs a mismatch error: { the noise level most consistent with the image (the ML estimate) diverges from the noise level dictated by the schedule of the algorithm.} Figure~\ref{fig:mismatch_error_celeba} illustrates the nature of the mismatch error in one-shot denoising in a non-blind model. Each plot shows the mean-squared error (MSE) between denoised and clean images, as a function of the second argument. MSE is lowest when $\sigma = \sigma^\star$. If $\sigma$ is too small, the denoised image is still noisy; too large, and it looks blurry. 
\begin{figure}
    \centering
   \includegraphics[width=.99\linewidth]{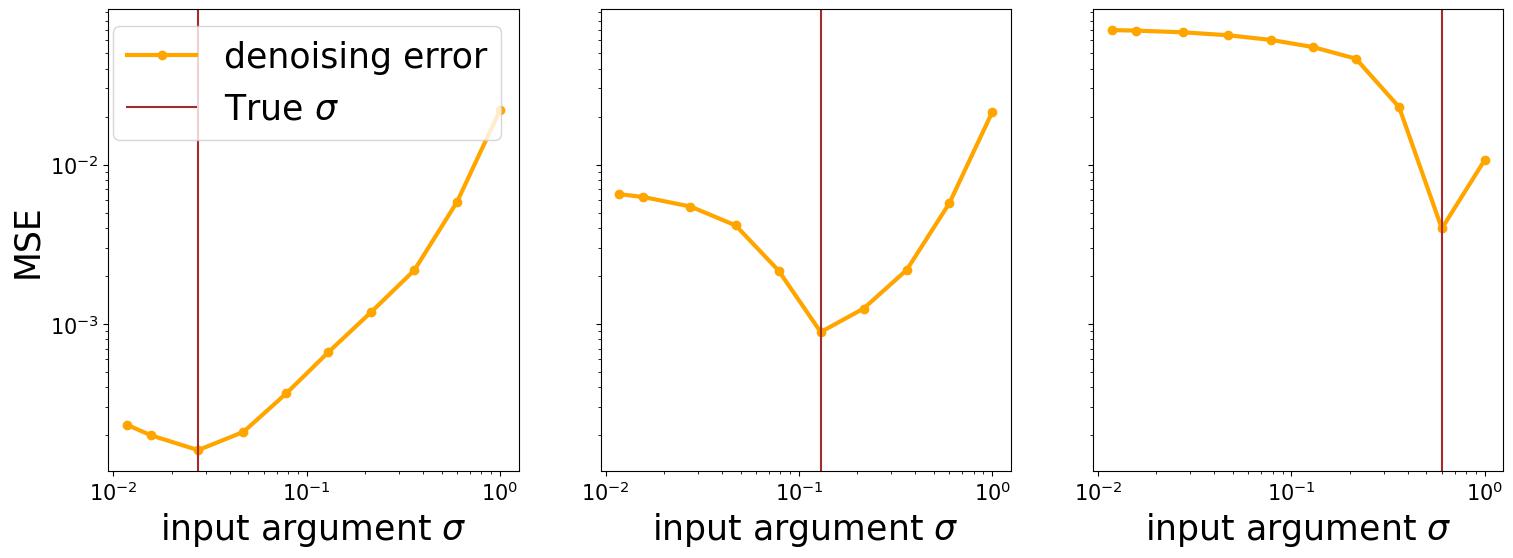}     
   \includegraphics[width=.26\linewidth]{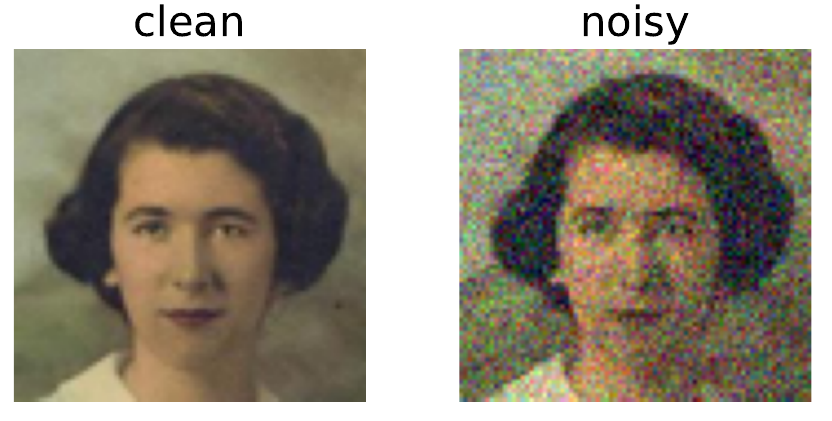}
   \hfill
   \includegraphics[width=.70\linewidth]{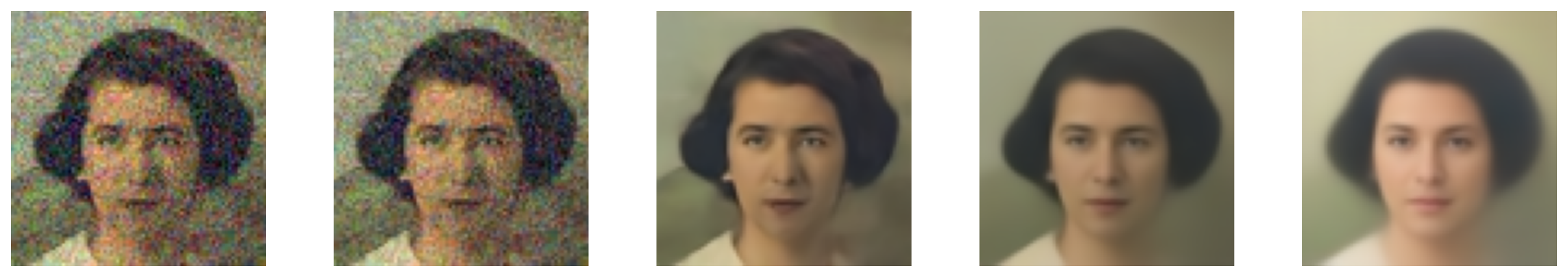}        
     \caption{Sensitivity of non-blind denoising performance to mismatches in true noise level and argument to the denoiser.  
        {\textbf{Top row:} Denoising error (averaged over 512 samples) as a function of argument $\sigma$ for three different true noise levels (from left to right, $\sigma^\star\in[0.025, 0.15, 0.6])$.
        \textbf{Bottom row:} Example clean image $x_0$, noisy image $x_{\sigma^\star} = x_0 + \sigma^\star z$, for $\sigma^\star= 0.1$,  and five images resulting from a non-blind denoiser $\widetilde{x} = \widetilde f_\theta(x_{\sigma^\star},\sigma)$ with $\sigma \in [0.01, 0.03, 0.10, 0.32, 1.0]$ (from left to right). 
        }}
    \label{fig:mismatch_error_celeba}
\end{figure}

To quantify the mismatch in the DDPM backward sampling process, we train a separate small neural network to estimate noise level. Figure~\ref{fig:noise-estimator-model} in the appendix shows that this model can nearly perfectly estimate the noise level. 
We now use this model to measure the true noise level of the intermediate samples generated by DDPM, and compare them to the $\sigma$ imposed by the schedule.
Figure~\ref{fig:mismatch_schedule} shows that scheduled noise levels systematically fall behind the true $\sigma$ value. This suggests that the low quality of the samples in Figures~\ref{fig:samples_celeba_1k}, \ref{fig:samples_celeba_17k}, and \ref{fig:samples_bed} result from a mismatch error in which $\sigma_t > \sigma^\star$ along the trajectory. Per the rightmost plot in Figure~\ref{fig:mismatch_schedule}, this same noise estimator model shows that {BDDMs precisely track the implicit schedule of the reverse process}.

\section{Conclusion}
This work presents a comprehensive mathematical analysis and justification of blind denoising diffusion models, closing a gap in the literature. We identify \emph{low intrinsic dimensionality} of the underlying data distribution as a critical component underlying the success of these models both theoretically and empirically. Moreover, we have shown that BDDMs can offer improved sample diversity and quality by avoiding errors due to mismatch in noise schedules. The use of BDDMs merits further investigation at larger scales, as well as in application to other downstream tasks such as fine-tuning and inverse problems.

\section*{Acknowledgments}
ZK acknowledges the computing facilities of the Flatiron Institute. AAP acknowledges the Yale Institute for the Foundations of Data Science for financial support.

\bibliography{example_paper}

\newpage
\appendix
\section{Proofs for Sections~\ref{sec:theory_optimal} and \ref{sec:theory_implicit_noise}}\label{app:proofs_3132}
\subsection{Proof of Proposition~\ref{prop:bd_minimizer}}\label{app:bd_minimizer_proof}
The population counterpart to~\eqref{eq:empirical_loss} is
\begin{align*}
    f^\star = \argmin_{f : \R^d\to\R^d} \E\| X - f( X + \sigma z)\|^2\,,
\end{align*}
where the expectation is under ${ X\sim \pdata}$, $\sigma \sim \prior$, and $z\sim\cN(0, I)$.
Let $y \defeq X +\sigma z$.
The optimal function $f$ of $y$ to predict $X$ is the conditional expectation $\E[X|y]$, due to orthogonality:
\begin{align}\label{eq:orthogonality}
    \E\|X - f(y)\|^2
    &= \E\|X - \E[X| y]\|^2 + \E\|\E[X | y] - f(y)\|^2\,.
\end{align}
On the other hand, by Bayes, we can express
\begin{align*}
    \E[X| y]
    &= \int \E[X | \sigma, y] \dd\post(\sigma| y)\,.
\end{align*}
Applying Tweedie's identity~\eqref{eq:tweedie},
\begin{align*}
    \E[X| y]
    &= \int \bigl(y+\sigma^2\,\nabla \log p_\sigma(y)\bigr) \dd\post(\sigma| y)\,.
\end{align*}
We remark that by~\eqref{eq:orthogonality}, for any estimator $\widehat f$, the excess risk is
\begin{align}\label{eq:excess_risk}
\begin{split}
    \cE(\widehat f) &\defeq \E\|X - \widehat f(y)\|^2 - \E\|X - f^\star(y)\|^2\\
    &= \E\|\widehat f(y) - f^\star(y)\|^2 = \int \|\widehat f-f^\star\|_{L^2(p_\sigma)}^2\dd\prior(\sigma)\,.
\end{split}
\end{align}
\subsection{Proof of Proposition~\ref{prop:opt_schedule}}\label{app:opt_schedule}
Let $\rho_t \defeq {\msf{Law}}(X_t)$ and $\tilde\rho_t \defeq p_{\sigma_t} \defeq p_X * \gamma_{\sigma_t^2}$, where we abbreviate $\gamma_{\sigma^2} \defeq \cN(0,\sigma^2 I)$ for arbitrary $\sigma > 0$. We first compute directly $\partial_t\tilde\rho_t$ from this definition using the fact that the Gaussian density is the solution to the heat equation:
\begin{align}
    \partial_t \tilde\rho_t(y) &= \partial_t \int \gamma_{\sigma_t^2}(y)\,p_X(y-x)\dd x\nonumber \\
    &=\int (\partial_t \gamma_{\sigma_t^2}(y))\,p_X(y-x)\dd x\nonumber\\
    &= \int (\partial_t \sigma_t^2)\, (\partial_s\gamma_{s}|_{s = \sigma_t^2}(y))\,p_X(y-x)\dd x \nonumber \\[0.25em]
    &= \frac{1}{2}\,(\partial_t \sigma_t^2) \int\Delta\gamma_{\sigma_t^2}(y)\,p_X(y-x)\dd x\nonumber \\[0.25em]
    &= \frac{1}{2}\,(\partial_t\sigma_t^2)\, \Delta_y \int \gamma_{\sigma_t^2}(y)\,p_X(y-x)\dd x\nonumber \\[0.25em]
    &= \frac{1}{2}\,(\partial_t\sigma_t^2)\,\Delta \tilde\rho_t(y)\,.\label{eq:tilde_rho_eq}
\end{align}
On the other hand, under the ideal dynamics \eqref{eq:ideal_sde}, the Fokker--Planck equation reads
\begin{align}
    \partial_t \rho_t
    &= -\nabla\cdot(\rho_t \sigma_t^2\nabla \log p_{\sigma_t}) + a_t\, \Delta \rho_t \nonumber\\
    &= -\nabla\cdot(\rho_t \sigma_t^2\nabla \log \tilde\rho_t) + a_t\, \Delta \rho_t\,.\label{eq:rho_eq}
\end{align}
Since $\Delta\tilde\rho_t = \nabla \cdot(\tilde\rho_t\nabla \log \tilde\rho_t)$,~\eqref{eq:tilde_rho_eq} shows that $(\tilde\rho_t)_{t\in [0,T]}$ solves the equation~\eqref{eq:rho_eq}, provided that
\begin{align}\label{eq:sigmat_ode}
    \frac{1}{2}\,\partial_t (\sigma_t^2) = -\sigma_t^2+a_t\,.
\end{align}
By uniqueness of solutions to the Fokker--Planck equation, we see that with this choice of $(\sigma_t)_{t\in [0,T]}$ and for $\rho_0 = \tilde\rho_0$, we have $\rho_t = \tilde\rho_t$ for all $t\in [0,T]$.
Solving the ODE~\eqref{eq:sigmat_ode} yields the claim.

\subsection{Proof of Lemma~\ref{lem:decr_noise}}

Recall that $\frac{1}{2}\,\partial_t(\sigma_t^2) = -\sigma_t^2 + a_t$, so we want to show $a_t \le \sigma_t^2$ for all $t\ge 0$. Note that if $a_t \le a_0 \le \sigma_0^2$ for all $t\ge 0$, then
\begin{align*}
    \sigma_t^2
    &= \sigma_0^2\, e^{-2t} + 2\int_0^t a_s\, e^{-2(t-s)}\dd s
    \ge \sigma_0^2\, e^{-2t} + a_t\,(1-e^{-2t})\,.
\end{align*}
This is lower bounded by $a_t$ provided $a_t \le \sigma_0^2$.

\subsection{Proof of Lemma~\ref{lem:noise_to_zero}}
The first term in Proposition~\ref{prop:opt_schedule} is obviously decaying to zero, so we must show that $2\int_0^t a_s e^{-2(t-s)}\dd s \to 0$.
Since $a_t \to 0$, for any $\varepsilon > 0$ there exists $t_0$ such that $a_t \le \varepsilon$ for all $t\ge t_0$.
Then, for $t\ge t_0$,
\begin{align*}
    \int_0^t 2a_s e^{-2(t-s)}\dd s
    &\le \int_{0}^{t_0}(\cdots) + \int_{t_0}^t(\cdots) \\
    &\le a_0 \,(e^{-2(t-t_0)} - e^{-2t}) + \varepsilon\,(1-e^{-2(t-t_0)})\\
    &\le a_0 \,(e^{-2(t-t_0)} - e^{-2t}) + \varepsilon\,.
\end{align*}
We can choose $t$ sufficiently large to make the first term at most $\varepsilon$ as well.

\section{Proofs for Sections~\ref{sec:main_result} and \ref{sec:theory_bayesian}}\label{app:proofs_33}

We use the following notation throughout the appendix. Recalling that $p_\sigma \defeq p * \cN(0,\sigma^2I)$, we write $q_\sigma$ for the joint distribution under which $X\sim \pdata$, $Y\sim \cN(X,\sigma^2 I)$. We also denote
\begin{align}
    f^\star_\sigma(y) &\defeq \E_{q_\sigma}[X\mid Y=y]\,,\label{eq:deff} \\
    C_\sigma(y) &\defeq \E_{q_\sigma}[(X - f^\star_\sigma(y))^{\otimes 2} \mid Y=y] = {\rm Cov}_{q_\sigma}(X\mid Y=y)\,, \label{eq:defC} \\
    W_\sigma(y) &\defeq {\rm Cov}_{q_\sigma}(X,\|X-y\|^2\mid Y=y)\,. \label{eq:defW}
\end{align}
Also, since we repeatedly encounter the quantity $\msf k + \log(1/\delta)$ for some $\delta > 0$, we abbreviate this by $\msf k_\delta$.

\subsection{Proof of Lemma~\ref{lem:log-derivatives}}\label{proof:log-derivatives}
The change-of-variables formula is straightforward as $|\frac{\dd \sigma}{\dd \lambda}| = \frac{1}{2}\,\lambda^{-3/2}$. For the next part, we write
\begin{align*}
    \ell(\lambda|y) = -\log \ppost(\lambda|y) = -\frac{d-\alpha}{2}\log\lambda - \log\E_{X\sim p_X}\bigl[\exp\bigl(-\frac{\lambda}{2}\,\|X-y\|^2\bigr)\bigr] + c\,,
\end{align*}
where $c > 0$ is an absolute constant. We compute via chain rule
\begin{align*}
    \ell'(\lambda|y) &= 
    -\frac{d-\alpha}{2\lambda} + \frac{1}{2}\,\E_{q_\lambda}[\|X-Y\|^2\mid Y=y]\,,
\end{align*}
where $q_{\lambda}(x|y) \propto \exp\bigl(-\frac{\lambda}{2}\,\|x-y\|^2)\,p_X(x)$. Using the following fact (which can be verified via chain rule)
\begin{align*}
    \partial_\lambda \E_{q_\lambda}[\|X-y\|^2\mid Y=y] = -\frac{1}{2}\var(\|X-y\|^2\mid Y=y)\,,
\end{align*}
the second derivative is computed similarly, resulting in
\begin{align*}
\ell''(\lambda|y) &= \frac{d-\alpha}{2\lambda^2} - \frac{1}{4}\,{\rm Var}_{q_\lambda}[\|X-Y\|^2\mid Y=y]\,.
\end{align*}

\subsection{Proof of Proposition~\ref{prop:nut_informal}}

In this section, we use the notation
\begin{align*}
    \msf e_\delta \defeq \sqrt{\frac{\log(1/\delta)}{d}} + \frac{\msf k_\delta}{d}\,.
\end{align*}
We will show that the noise posterior concentrates up to a relative error of $\msf e_\delta$.

\begin{theorem}\label{thm:noise_formal}
    Let $\delta \in (0,1)$ and assume that $d \gg \msf k_\delta$.
    Then, the following holds with probability at least $1-\delta$ over $X_t \sim p_{\sigma_t}$.
    For all $a \gg \lambda_t \msf e_\delta$,
    \begin{align*}
        \ppost(\abs{\lambda-\lambda_t}\ge a\mid X_t)
        &\lesssim \frac{\msf e_\delta}{\sqrt d} \exp\bigl(-\Omega\bigl(\frac{da^2}{\lambda_t^2}\bigr)\bigr) + \frac{\lambda_{\max}}{\lambda_t \msf e_\delta} \exp(-\Omega(d))\,.
    \end{align*}
\end{theorem}
\begin{proof}
We apply Proposition~\ref{prop:helper_prop}, noting that
\begin{align*}
    \E_{q_\lambda}[\|X-Y\|^2\mid Y=y]
    &= D_{\lambda^{-1/2},2}(y,y)\,.
\end{align*}
Hence, with probability at least $1-\delta$,
\begin{align*}
    2\ell'(\lambda| X_t)
    &= -\frac{d-\alpha}{\lambda} + D_{\lambda^{-1/2},2}(X_t,X_t) \\
    &= -\frac{d \pm O(\msf k_\delta)}{\lambda} + \frac{d \pm O(\sqrt{d\log(1/\delta)} + \msf k_\delta))}{\lambda_t}\,.
\end{align*}
Define the error term
\begin{align*}
    \msf E \defeq \sqrt{d\log(1/\delta)} + \msf k_\delta\,.
\end{align*}
Note that for $\lambda > \lambda_t$,
\begin{align*}
    2\ell'(\lambda|X_t)
    &\ge \frac{d}{\lambda_t}\,\Bigl[1 - \frac{\lambda_t}{\lambda} - \frac{O(\msf E)}{d} \Bigr]
    \gtrsim \begin{cases}
        d\,(\lambda-\lambda_t)/\lambda_t^2\,, & C_0\lambda_t \msf E/d \le \lambda-\lambda_t \le \lambda_t\,, \\
        d/\lambda_t\,, & \lambda-\lambda_t \ge\lambda_t\,,
    \end{cases}
\end{align*}
where $C_0 > 0$ is a universal constant and
we require that $\msf E/d \ll 1$, i.e., $d \gg \msf k_\delta$. Therefore, for $a_0 \defeq C_0\lambda_t \msf E/d$ and $\lambda \ge \lambda_t + 2a_0$, we can integrate to find that
\begin{align*}
    \ppost(\lambda|X_t)
    &= \ppost(\lambda_t+a_0|X_t)\exp\Bigl(-\int_{\lambda_t+a_0}^\lambda \ell'(\cdot|X_t)\Bigr) \\
    &\le \ppost(\lambda_t+a_0|X_t)\times \begin{dcases}
        \exp\bigl(-\Omega\bigl(\frac{d\,(\lambda-\lambda_t)^2}{\lambda_t^2}\bigr)\bigr)\,, & \lambda_t + 2a_0 \le \lambda \le 2\lambda_t\,, \\
        \exp(-\Omega(d))\,, & \lambda \ge 2\lambda_t\,.
    \end{dcases}
\end{align*}
Hence, for $a \ge 2a_0$,
\begin{align*}
    \int_{\lambda_t+a}^\infty \dd\ppost(\lambda|X_t)
    &\le \int_{\lambda_t+a}^{2\lambda_t} \ppost(\lambda_t+a_0| X_t) \exp\bigl(-\Omega\bigl(\frac{d\,(\lambda-\lambda_t)^2}{\lambda_t^2}\bigr)\bigr)\dd \lambda \\
    &\quad\quad\quad+ \int_{2\lambda_t}^{\lambda_{\max}} \ppost(\lambda_t+a_0 | X_t) \exp(-\Omega(d))\dd\lambda \\
    &\lesssim \Bigl[\frac{\lambda_t}{d^{1/2}} \exp\bigl(-\Omega\bigl(\frac{da^2}{\lambda_t^2}\bigr)\bigr) + \lambda_{\max}\exp(-\Omega(d))\Bigr]\,\ppost(\lambda_t+a_0| X_t)\,.
\end{align*}
On the other hand, for $\lambda\le \lambda_t + a_0$, one has $|\ell'(\lambda| X_t)| \lesssim \msf E/\lambda_t$.
Hence,
\begin{align*}
    \ppost(\lambda_t+a_0| X_t)
    &= \frac{1}{a_0} \int_{\lambda_t}^{\lambda_t+a_0} \ppost(\lambda_t+a_0| X_t)\dd \lambda\\
    &\lesssim \frac{d}{\lambda_t\msf E} \int_{\lambda_t}^{\lambda_t+a_0} \ppost(\lambda| X_t) \exp\bigl(O\bigl(\frac{a_0 \msf E}{\lambda_t}\bigr)\bigr)\,\D \lambda \\
    &\le \frac{d}{\lambda_t \msf E}\exp\bigl(O\bigl(\frac{\msf E^2}{d}\bigr)\bigr)\,,
\end{align*}
and thus
\begin{align*}
    \int_{\lambda_t+a}^\infty \dd\ppost(\lambda| X_t)
    &\lesssim  \frac{d^{1/2}}{\msf E} \exp\bigl(O\bigl(\frac{\msf E^2}{d}\bigr) - \Omega\bigl(\frac{da^2}{\lambda_t^2}\bigr)\bigr) + \frac{\lambda_{\max} d}{\lambda_t \msf E} \exp\bigl(O\bigl(\frac{\msf E^2}{d}\bigr) - \Omega(d)\bigr)\,.
\end{align*}
Since we are assuming that $\msf E\ll d$, the last exponential is $\exp(-\Omega(d))$.
For all $a \gg \lambda_t \msf E/d$, the bound becomes
\begin{align*}
    \int_{\lambda_t+a}^\infty \dd\ppost(\lambda| X_t)
    &\lesssim \frac{d^{1/2}}{\msf E} \exp\bigl(- \Omega\bigl(\frac{da^2}{\lambda_t^2}\bigr)\bigr) + \frac{\lambda_{\max} d}{\lambda_t \msf E} \exp(-\Omega(d))\,.
\end{align*}
A similar argument gives the corresponding lower bound.
\end{proof}

Integrating this high probability bound yields a bound in mean-squared error.

\begin{proof}[Proof of Proposition~\ref{prop:nut_informal}]
    Work over the event of probability at least $1-\delta$ under which Theorem~\ref{thm:noise_formal} holds.
    Then, provided that $d \gg \msf k_\delta$, for $\kappa \defeq \lambda_{\max}/\lambda_{\min}$,
    \begin{align*}
        &\int |\lambda-\lambda_t|^2\dd\ppost(\lambda|X_t)
        \lesssim \lambda_t^2 \msf e_\delta^2 + \int_{|\lambda-\lambda_t|\ge C\msf e_\delta} |\lambda-\lambda_t|^2\dd\ppost(\lambda|X_t) \\[0.25em]
        &\qquad \lesssim \lambda_t^2 \msf e_\delta^2 + \int_0^{\lambda_{\max}} a\,\Bigl(\frac{\msf e_\delta}{\sqrt d} \exp\bigl(-\Omega\bigl(\frac{da^2}{\lambda_t^2}\bigr)\bigr) + \frac{\lambda_{\max}}{\lambda_t \msf e_\delta} \exp(-\Omega(d))\Bigr)\dd a \\[0.25em]
        &\qquad \lesssim \lambda_t^2\,\Bigl[ \msf e_\delta^2 + \frac{\msf e_\delta}{d^{3/2}} + \frac{\lambda_{\max}^3}{\lambda_t^3 \msf e_\delta} \exp(-\Omega(d))\Bigr] \\[0.25em]
        &\qquad \lesssim \lambda_t^2 \,\Bigl[ \frac{\log(1/\delta)}{d} + \frac{\msf k_\delta^2}{d^2} + \frac{\sqrt{\log(1/\delta)}}{d^2} + \frac{\msf k_\delta}{d^{5/2}} + \kappa^3 \,\bigl(\sqrt d + \frac{d}{\msf k}\bigr) \exp(-\Omega(d))\Bigr]
        \lesssim \lambda_t^2 \msf e_\delta^2\,,
    \end{align*}
    provided that $d \gg \log \kappa$.

    In conclusion, for all $\delta \in (0,1)$, if $d \ge C\,(\msf k + \log(1/\delta) + \log (\sigma_0/\sigma_T))$ for a sufficiently large absolute constant $C > 0$, then with probability at least $1-\delta$ it holds that
    \begin{align}\label{eq:noise_formal}
        \int |\lambda-\lambda_t|^2\dd\ppost(\lambda|X_t)
        &\lesssim \lambda_t^2\,\Bigl(\frac{\log(1/\delta)}{d} + \frac{\msf k^2 + \log^2(1/\delta)}{d^2}\Bigr)\,. 
    \end{align}
\end{proof}

\subsection{Proof of Theorem~\ref{thm:thm_main}}

Recall the KL error decomposition~\eqref{eq:girsanov_33}. For the first term, we use the standard fact
\begin{align*}
    \KL(p_{\sigma_0}\|\widehat p_0)
    &= \KL(\pdata * \gamma_{\sigma_0^2}\|\gamma_{\sigma_0}^2)
    \le \frac{1}{2\sigma_0^2}\,W_2^2(\pdata, \delta_0)
    = \frac{\E_{X\sim\pdata}[\|X\|^2]}{2\sigma_0^2}
    \lesssim \frac{\msf m^2_2}{\sigma_0^2}\,.
\end{align*}
The second term is due to the definition of $\epsBD^2$, and now we focus our attention on controlling the third term.

We work over the event of probability at least $1-\delta$ under which both~\cref{thm:noise_formal} and~\cref{cor:helper_cor} hold.
Then,
\begin{align}
    \Bigl|\int_\sigma^{\sigma_t}\|W_{\omega}(X_t)\|\, \omega^{-3}\dd \omega\Bigr|^2 &\lesssim \msf k_\delta^3 \,\Bigl|\int_\sigma^{\sigma_t} \sigma_t^3\, \omega^{-3} \dd \omega + \int_\sigma^{\sigma_t} \dd \omega \Bigr|^2 \label{eq:someWbd} \\[0.25em]
    &\lesssim \sigma_t^2\,\msf k_\delta^3 \,\bigl(|1 - \sigma_t^2/\sigma^2|^2 + |1-\sigma/\sigma_t|^2\bigr) \nonumber \\[0.25em]
    &= \frac{\msf k_\delta^3}{\lambda_t}\,\bigl(|1-\lambda/\lambda_t|^2 + |1-\sqrt{\lambda_t/\lambda}|^2\bigr)\,. \nonumber
\end{align}
Now, we integrate this w.r.t.\ $\ppost(\cdot|X_t)$.
By~\eqref{eq:noise_formal}, if $d \gg \msf k_\delta + \log\kappa$,
\begin{align*}
    &\int |1-\lambda/\lambda_t|^2\dd\ppost(\lambda|X_t)
    \lesssim \msf e_\delta^2\,.
\end{align*}
Similarly, 
\begin{align*}
    &\int |1-\sqrt{\lambda_t/\lambda}|^2 \dd \ppost(\lambda|X_t) \\
    &\qquad \lesssim \msf e_\delta^2 + \frac{1}{\lambda_t^2} \int_{C\msf e_\delta \le |\lambda-\lambda_t| \le \lambda_t/2} |\lambda-\lambda_t|^2 + \frac{\lambda_t}{\lambda_{\min}}\,\ppost(|\lambda-\lambda_t| \ge \lambda_t/2 \mid X_t) \\
    &\qquad \lesssim \msf e_\delta^2 + \frac{\lambda_t}{\lambda_{\min}}\,\bigl(\frac{\msf e_\delta}{\sqrt d} + \frac{\lambda_{\max}}{\lambda_t \msf e_\delta}\bigr) \exp(-\Omega(d)) \lesssim \msf e_\delta^2\,.
\end{align*}
We have shown that for all $\delta$ such that $d \gg \msf k_\delta + \log\kappa$, with probability at least $1-\delta$,
\begin{align*}
    \|(r_{\sigma_t}^\star - r^\star)(X_t)\|^2
    &\lesssim \frac{\msf k_\delta^3}{\lambda_t}\,\msf e_\delta^2\,.
\end{align*}
On the other hand, for smaller values of $\delta$, from~\eqref{eq:someWbd} we still have the crude bound
\begin{align*}
    \|(r_{\sigma_t}^\star - r^\star)(X_t)\|^2
    &\lesssim \frac{\msf k_\delta^3}{\lambda_t}\,\kappa^4\,.
\end{align*}
Therefore, for a universal constant $c > 0$, we deduce that for \emph{all} $\delta \in (0,1)$, if $d \gg \msf k + \log \kappa$,
\begin{align*}
    \|(r_{\sigma_t}^\star - r^\star)(X_t)\|^2
    &\lesssim \frac{\msf k_\delta^3}{\lambda_t}\,\bigl[\msf e_\delta^2 + \kappa^4\, \mb 1_{\delta \le \exp(-cd)}\bigr]\,.
\end{align*}

Integrating this high probability bound yields
\begin{align*}
    \|r_{\sigma_t}^\star - r^\star\|_{L^2(p_{\sigma_t})}^2
    &\lesssim \frac{\msf k^3}{\lambda_t}\,\Bigl[ \frac{1}{d} + \frac{\msf k^2}{d^2} + \kappa^4\exp(-\Omega(d))\Bigr]
    \lesssim \frac{\msf k^3}{\lambda_t}\,\bigl( \frac{1}{d} + \frac{\msf k^2}{d^2}\bigr)\,.
\end{align*}
Thus, the noise estimation term becomes
\begin{align*}
    \int_0^T \frac{1}{a_t}\,\|r_{\sigma_t}^\star - r^\star\|_{L^2(p_{\sigma_t})}^2\dd t
    &\lesssim \bigl(\frac{\msf k^3}{d} + \frac{\msf k^5}{d^2}\bigr) \int_0^T \frac{\sigma_t^2}{a_t}\dd t\,.
\end{align*}

\section{Proofs for Section~\ref{sec:theory_discretization}}\label{app:proofs_36}
\subsection{Exponential Euler discretization}\label{app:exp_euler}
In order to establish discretization guarantees which only scale with the intrinsic dimension, we use the exponential Euler discretization: for fixed step size $h > 0$, the algorithm follows
\begin{align}\label{eq:disc_sde_1}
    \dd Y_t = (\widehat f(Y_{t_-}) - Y_t)\dd t + \sqrt{2a_t}\dd B_t\,,
\end{align}
where $t_- \defeq \lfloor t/h\rfloor\,h$.
In other words, we freeze the non-linear term $\widehat f$ and exactly integrate the rest, including the linear term.
Some intuition for this can be seen as follows: if, instead of $\widehat f$, we had $f_{\sigma_t}^\star$, then the Jacobian of the drift is $C_{\sigma_t}/\sigma_t^2-I$.
Generally, discretization error bounds rely on Lipschitz bounds on the drift, i.e., bounds on the operator norm of this Jacobian.
However, among the two terms in the Jacobian, only the first---a conditional covariance matrix---is expected to be controlled in terms of the intrinsic dimension.
This suggests that to avoid dependence on the ambient dimension, we should exactly integrate the $-Y_t$ term.

\subsection{Proof of Theorem~\ref{thm:disc_error}}

Recall that 
$f_{\sigma_t}^\star = {\rm id} + \sigma_t^2\, \nabla \log p_{\sigma_t}$. 
To control the discretization error, we first need an It\^o calculation.

\begin{lemma}\label{lem:disc_lemma1}
Let $(\sigma_t)_{t\ge 0}$ be as in Proposition~\ref{prop:opt_schedule}. Then
\begin{align*}
    \partial_t f_{\sigma_t}^\star(y) = \dot\sigma_t\, \partial_\omega(\omega^2\, \nabla \log p_\omega(y))\big|_{\omega = \sigma_t} = \Bigl(\frac{a_t}{\sigma_t^4} - \frac{1}{\sigma_t^2}\Bigr)\,W_{\sigma_t}(y) \,.
\end{align*}
\end{lemma}
\begin{proof}
Follows from the chain rule and the final expression for $\dot\sigma_t$ from \eqref{eq:sigmat_ode}. 
\end{proof}

\begin{proposition}\label{prop:ito_prop}
For $(X_t)_{t\ge 0}$ following \eqref{eq:ideal_sde}, it holds that
\begin{align*}
    \dd f_{\sigma_t}^\star(X_t)
    &= \Bigl\{\frac{2a_t - \sigma_t^2}{\sigma_t^4}\,\bigl(W_{\sigma_t}(X_t) - C_{\sigma_t}(X_t)\,(f_{\sigma_t}^\star(X_t) - X_t)\bigr)\Bigr\}\dd t \\
    &\qquad{} +\sqrt{2a_t}\sigma_{t}^{-2}\,C_{\sigma_t}(X_t)\dd B_t\,.
\end{align*}
\end{proposition}
\begin{proof}
Writing out It\^o's lemma,
\begin{align*}
    \dd f^\star_{\sigma_t}(X_t) = (\partial_t f^\star_{\sigma_t})(X_t)\dd t + \langle \nabla f^\star_{\sigma_t}(X_t), \dd X_t \rangle + \frac{1}{2}\,\langle \dd X_t, \nabla^2 f^\star_{\sigma_t}(X_t) \dd X_t \rangle\,,
\end{align*}
we see that the only terms that remain are
\begin{align*}
    \dd f^\star_{\sigma_t}(X_t) &= \big[(\partial_t f^\star_{\sigma_t})(X_t) + \nabla f^\star_{\sigma_t}(X_t)\, (f^\star_{\sigma_t}(X_t) - X_t)  + a_t\,\Delta f^\star_{\sigma_t}(X_t) \bigr]\dd t \\
    &\qquad\qquad\qquad+ \sqrt{2a_t}\, \nabla f^\star_{\sigma_t}(X_t) \dd B_t\,.
\end{align*}
From Lemma~\ref{lem:disc_lemma1} we have that
\begin{align*}
    \partial_t f^\star_{\sigma_t}(y) = \frac{a_t - \sigma_t^2}{\sigma_t^4}\,W_{\sigma_t}(y)
\end{align*}
and by Tweedie's formula (recall Lemma~\ref{lem:tweedie_formula})
\begin{align*}
    a_t\,\Delta f^\star_{\sigma_t} = a_t \sigma_t^2\, \Delta \nabla \log p_{\sigma_t} = a_t \sigma_t^2\, \nabla \Delta \log p_{\sigma_t} = a_t \sigma_t^{-2}\, \nabla \tr C_{\sigma_t}\,,
\end{align*}
where the identity matrix drops due to the additional gradient. Finally, note that 
\begin{align*}
    \nabla f^\star_{\sigma_t}(y)\, (f^\star_{\sigma_t}(y) - y) = \sigma_t^{-2}\,C_{\sigma_t}(y)\,(f^\star_{\sigma_t}(y) - y)\,.
\end{align*}
Collecting these three terms and invoking the last result in Lemma~\ref{lem:tweedie_formula}, we obtain
\begin{align*}
    &\bigl((\partial_t f^\star_{\sigma_t}) + \nabla f^\star_{\sigma_t}\,(f^\star_{\sigma_t} - {\rm id})  + a_t\,\Delta f^\star_{\sigma_t}\bigr)(y)\\
    &\phantom{=}= \frac{a_t - \sigma_t^2}{\sigma_t^4}\,W_{\sigma_t}(y) + a_t\sigma_t^{-2}\,\nabla \tr C_{\sigma_t}(y) + \sigma_t^{-2}\,C_{\sigma_t}(y)\,(f^\star_{\sigma_t}(y) - y) \\
    &\phantom{=}=\Bigl( \frac{a_t - \sigma_t^2}{\sigma_t^4}\Bigr)\, W_{\sigma_t}(y)  + \Bigl( \frac{a_t}{\sigma_t^4}\, W_{\sigma_t}(y) - \frac{2a_t}{\sigma_t^4}\,C_{\sigma_t}(y)\,(f^\star_{\sigma_t}(y)-y)\Bigr) \\
    &\phantom{=}\qquad\qquad+ \sigma_t^{-2}\,C_{\sigma_t}(y)\,(f^\star_{\sigma_t}(y) - y) \\
    &\phantom{=}= \frac{2a_t - \sigma_t^2}{\sigma_t^4}\,\bigl(W_{\sigma_t}(y) - C_{\sigma_t}(y)\,(f^\star_{\sigma_t}(y) - y)\bigr)\,. \qedhere
\end{align*}
\end{proof}

\begin{lemma}\label{lem:mmse}
    For $X_\sigma \sim p_\sigma$,
    \begin{align*}
        \partial_\sigma\, \E \tr C_\sigma(X_\sigma) = \frac{2}{\sigma^3} \,\E\tr\bigl(C_\sigma(X_\sigma)^2\bigr)\,.
    \end{align*}
\end{lemma}
\begin{proof}
    This is a variant of the standard MMSE identity.
    For example,~\citet[\S II.C]{ElAMon22StoLoc} with $\bs Q = \bs R = I$ shows that
    \begin{align*}
        \partial_\lambda\, \E\tr C_{1/\sqrt\lambda}(X_{1/\sqrt\lambda}) = -\E\tr\bigl(C_{1/\sqrt\lambda}(X_{1/\sqrt\lambda})^2\bigr)\,,
    \end{align*}
    with $\lambda \defeq \sigma^{-2}$.
    The claimed result follows from the chain rule.
\end{proof}

We now present the main computations from Section~\ref{sec:theory_discretization}. Applying Girsanov's theorem to the ideal process \eqref{eq:ideal_sde}, denoted $\msf P$, and the exponential Euler discretization~\eqref{eq:disc_sde_1}, denoted $\widehat{\msf P}$, the error becomes
\begin{align}
    \!\!\KL(\msf P \|\widehat{\msf P})
    &\lesssim \KL(p_0\|\widehat p_0)
    + \E\int_0^T \frac{1}{a_t}\, \|\widehat f_\theta(X_{t_-}) - f^\star_{\sigma_t}(X_t)\|^2\dd t \nonumber\\
    &\lesssim \KL(p_0\|\widehat{ p}_0)
    + \E\int_0^T \frac{1}{a_t}\, \|\widehat f_\theta(X_{t_-}) - f^\star(X_{t_-})\|^2 \dd t\nonumber\\
    &\qquad\qquad+ \E\int_0^T {\frac{1}{a_t}}\, \bigl(\|f^\star(X_{t_-}) - f^\star_{\sigma_{t_-}}(X_{t_-})\|^2+\|f^\star_{\sigma_{t_-}}(X_{t_-}) - f^\star_{\sigma_{t}}(X_t)\|^2\bigr)\dd t \nonumber\\
    &\lesssim \KL(p_0\|\widehat{ p}_0)
    + \tepsBD^2 + \E\int_0^T \frac{1}{a_t}\,  \|f^\star(X_{t_-}) - f^\star_{\sigma_{t_-}}(X_{t_-})\|^2 \dd t \nonumber \\
    &\qquad\qquad+\E \int_0^T \frac{1}{a_t}\, \|f^\star_{\sigma_{t_-}}(X_{t_-}) - f^\star_{\sigma_{t}}(X_t)\|^2\dd t\nonumber\\
&\lesssim \frac{R^2}{\sigma_0^2}
    + \tepsBD^2 + \bigl(\frac{\msf k^3}{d} + \frac{\msf k^5}{d^2}\bigr)\int_0^T \frac{\sigma_t^2}{a_t}\dd t + \int_0^T \frac{1}{a_t}\,\E\|f^\star_{\sigma_{t_-}}(X_{t_-}) - f^\star_{\sigma_{t}}(X_t)\|^2\dd t\label{eq:girsanov_full_bd}
\end{align}
where we used the result from Theorem~\ref{thm:thm_main} in the last line.

For the last term, we have by Proposition~\ref{prop:ito_prop} and the It\^o isometry,
\begin{align*}
    &\E\|f^\star_{\sigma_{t_-}}(X_{t_-}) -f^\star_{\sigma_{t}}(X_t)\|^2
    \\
    &\quad= \E\Bigl\lVert \int_{t_-}^t \Bigl\{\frac{2a_s - \sigma_s^2}{\sigma_s^4}\,\bigl(W_{\sigma_s}(X_s) - C_{\sigma_s}(X_s)\,(f^\star_{\sigma_s}(X_s) - X_s)\bigr)\Bigr\}\dd s \\
    &\quad\qquad\qquad{} + \int_{t_-}^t\sqrt{2a_s}\sigma_{s}^{-2}\,C_{\sigma_s}(X_s)\dd B_s\Bigr\rVert^2 \\
    &\quad \lesssim h\int_{t_-}^t \bigl(\frac{2a_s - \sigma_s^2}{\sigma_s^4}\bigr)^2\,\E\|W_{\sigma_s}(X_s) - C_{\sigma_s}(X_s)\,(f^\star_{\sigma_{s}}(X_s) - X_s)\|^2\dd s \\
    &\quad\qquad\qquad+ \int_{t_-}^t a_s \sigma_s^{-4}\,\E \|C_{\sigma_s}(X_s)\|_{\rm F}^2\dd s\,.
\end{align*}

We now specialize to the class of noise schedules with $a_t = \ta\sigma_t^2$ for some $\ta \in (0,1)$.
In this case, we have $\sigma_t = \sigma_0 \exp(-(1-\ta)\,t)$ for $t\ge 0$, and
\begin{align*}
    &\E\|f^\star_{\sigma_{t_-}}(X_{t_-}) -f^\star_{\sigma_{t}}(X_t)\|^2 \\
    &\qquad \lesssim \bigl(\ta - \frac{1}{2}\bigr)^2\, h\int_{t_-}^t \frac{1}{\sigma_s^4}\,\E\|W_{\sigma_s}(X_s) - C_{\sigma_s}(X_s)\,(f^\star_{\sigma_{s}}(X_s) - X_s)\|^2\dd s \\
    &\qquad\qquad {}+ \ta \int_{t_-}^t \frac{1}{\sigma_s^2}\,\E \|C_{\sigma_s}(X_s)\|_{\rm F}^2\dd s\,.
\end{align*}
Note that for any non-negative $(b_t)_{t\in [0,T]}$, $T = Nh$,
\begin{align*}
    \int_0^T \frac{1}{a_t} \int_{t_-}^t b_s\dd s\dd t
    &= \sum_{n=0}^{(N-1)h} \int_{nh}^{(n+1)h} \frac{1}{a_t} \int_{nh}^t b_s\dd s\dd t \\
    &= \sum_{n=0}^{(N-1)h}\int_{nh}^{(n+1)h} b_s \int_s^{(n+1)h} \frac{1}{a_t}\dd t \dd s
    \lesssim h \int_0^T \frac{b_s}{a_s}\dd s\,,
\end{align*}
where we use the fact that $|\partial_t \log a_t| \lesssim 1$ and hence $a_s \asymp a_t$ for $\abs{s-t} \lesssim 1$.
Thus, the total discretization error becomes
\begin{align*}
    \msf{Disc}(h)
    &\defeq \int_0^T \frac{1}{a_t}\,\E\|f^\star_{\sigma_{t_-}}(X_{t_-}) - f^\star_{\sigma_{t}}(X_t)\|^2\dd t \\
    &\lesssim \frac{(\ta - \frac{1}{2})^2}{\ta}\, h^2\int_0^T \frac{1}{\sigma_s^6}\,\E\|W_{\sigma_s}(X_s) - C_{\sigma_s}(X_s)\,(f^\star_{\sigma_{s}}(X_s) - X_s)\|^2\dd s \\
    &\qquad{} + h\int_0^T \frac{1}{\sigma_s^4}\,\E \|C_{\sigma_s}(X_s)\|_{\rm F}^2\dd s\,.
\end{align*}

For the last term, note that by Lemma~\ref{lem:mmse},
\begin{align*}
    \partial_s \,\E\tr C_{\sigma_s}(X_s)
    = \frac{2\dot\sigma_s}{\sigma_s^3}\,\E\|C_{\sigma_s}(X_{\sigma_s})\|_{\rm F}^2
    = -\frac{2\,(1-\ta)}{\sigma_s^2}\,\E\|C_{\sigma_s}(X_{\sigma_s})\|_{\rm F}^2\,.
\end{align*}
Hence, the term can be written
\begin{align*}
    -\frac{h}{1-\ta} \int_0^T \frac{1}{\sigma_t^2}\,\partial_t\, \E \tr C_{\sigma_t}(X_t)\dd t
    &\le \frac{h}{1-\ta}\, \Bigl( \frac{\E \tr C_{\sigma_0}(X_0)}{\sigma_0^2} + \int_0^T \partial_t(\sigma_t^{-2})\,\E\tr C_{\sigma_t}(X_t) \dd t\Bigr) \\
    &\lesssim \frac{h}{1-\ta}\, \Bigl( \msf k - \int_0^T \frac{1}{\sigma_t^4}\, \partial_t \sigma_t^2\,\E\tr C_{\sigma_t}(X_t)\dd t\Bigr) \\
    &\lesssim \frac{h}{1-\ta}\, \Bigl( \msf k + (1-\ta) \int_0^T \frac{1}{\sigma_t^2}\,\E\tr C_{\sigma_t}(X_t) \dd t\Bigr) \\
    &\lesssim \bigl(\frac{1}{1-\ta} + T\bigr)\,\msf k h\,,
\end{align*}
where we applied Corollary~\ref{cor:helper_cor}.

As for the first term, by Corollary~\ref{cor:helper_cor}, we can bound 
\begin{align*}
    \E\|W_{\sigma_s}(X_s) - C_{\sigma_s}(X_s)\,(f^\star_{\sigma_{s}}(X_s) - X_s)\|^2
    &\lesssim \sigma_s^6\,\msf k^3\,.
\end{align*}
which gives 
\begin{align*}
    \int_0^T \frac{1}{\sigma_s^6}\,\E\|W_{\sigma_s}(X_s) - C_{\sigma_s}(X_s)\,(f^\star_{\sigma_{s}}(X_s) - X_s)\|^2\dd s
    &\lesssim \msf k^3 T\,.
\end{align*}

With this, we have our final bound
\begin{align*}
    \msf{Disc}(h)
    &\lesssim\frac{(\ta - \frac{1}{2})^2}{\ta}\, T\msf k^3 h^2 + \bigl(\frac{1}{1-\ta} + T\bigr)\,\msf k h\,.
\end{align*}
Note that $T \asymp \log(\sigma_0/\sigma_T)/(1-\ta)$, so this can be written
\begin{align*}
    \msf{Disc}(h)
    &\lesssim\frac{(\ta - \frac{1}{2})^2}{\ta\,(1-\ta)}\, \msf k^3 h^2 \log\frac{\sigma_0}{\sigma_T} + \frac{1}{1-\ta}\, \msf k h \log \frac{\sigma_0}{\sigma_T}\,.
\end{align*}

\section{Proof of Corollary~\ref{cor:sampling_cor}}\label{app:proofs_35}
First note that the choice $\sigma_T \asymp \eps/\sqrt{d}$ implies that $\msf D_{\rm BL}(p_{\sigma_T},p_X) \leq \eps$, as
\begin{align*}
    \msf D_{\rm BL}(p_{\sigma_T},p_X) \leq W_1(p_{\sigma_T},p_X) \leq W_2(p_{\sigma_T},p_X) \leq \sigma_T\sqrt d\,,
\end{align*}
where the last inequality follows from a standard coupling argument, where (for $p \in \{1,2\}$) $W_p$ are the $p$-Wasserstein distances. By the triangle inequality, we then have
\begin{align*}
    \msf D_{\rm BL}(\widehat p_T, p_X) \leq \msf D_{\rm BL}(p_{\sigma_T},p_X)  + \msf D_{\rm BL}(p_{\sigma_T},\widehat p_T) \leq \eps + \msf D_{\rm BL}(p_{\sigma_T},\widehat p_T) \,.
\end{align*}
To bound the remaining term, note that by Pinsker's inequality (as the total variation distance is an upper bound on $\msf D_{\rm BL}$), it suffices to obtain the bound 
\begin{align*}
    {\rm KL}(p_{\sigma_T}\|\widehat p_T)\lesssim \tilde\eps_{\rm BD}^2 + \eps^2\,.
\end{align*}
To this end, by~\eqref{eq:girsanov_full_bd} and \cref{thm:disc_error},
\begin{align*}
    \KL(p_{\sigma_T} \| \widehat p_T)
    &\lesssim \frac{R^2}{\sigma_0^2}
    + \tepsBD^2 + \bigl(\frac{\msf k^3}{d} + \frac{\msf k^5}{d^2}\bigr)\, T + \msf k h\,T\,,
\end{align*}
where we already used the choice $a_t = \tfrac12\,\sigma_t^2$. Noting that $T \asymp \log(\sigma_0/\sigma_T)$, we now choose $\sigma_0 \asymp \msf m_2/\varepsilon$, $h\asymp \varepsilon^2/(\msf k\log(\msf m_2\sqrt d/\varepsilon^2))$ and obtain
\begin{align*}
    \KL(p_{\sigma_T} \| \widehat p_T)
    &\lesssim \eps^2 + \tilde\eps^2_{\rm BD} +  \bigl(\frac{\msf k^3}{d} + \frac{\msf k^5}{d^2}\bigr) \log\frac{\msf m_2\sqrt{d}}{\eps^2}\,.
\end{align*}
The result concludes by invoking the remaining conditions.

\section{Proofs for Section~\ref{sec:prior}}\label{app:prior}
\subsection{Proof of Theorem~\ref{thm:score_error}}
By the definition of $\epsBD^2$ and from~\eqref{eq:excess_risk},
\begin{align*}
    \epsBD^2 = \int_0^T \frac{1}{a_t}\, \|\widehat f_\theta - f^\star\|_{L^2(p_{\sigma_t})}^2\dd t\,, \qquad \cE(\widehat f_\theta) = \int \|\widehat f_\theta - f^\star\|_{L^2(p_\sigma)}^2\dd\prior(\sigma)\,.
\end{align*}
Let us switch notation to $a(t) \defeq a_t$, $\sigma(t) \defeq \sigma_t$, and we perform the change of variables $\omega = \sigma(t)$ so that $\dd\omega = \dot\sigma(t)\dd t$.
Then, for $\eps(\sigma) \defeq \|\widehat f_\theta - f^\star\|_{L^2(p_\sigma)}$, the integral on the left can be written
\begin{align*}
    \epsBD^2
    &= \int_0^T \frac{\eps(\sigma(t))^2}{a(t)}\,\frac{1}{\dot\sigma(t)}\,\dot\sigma(t)\dd t
    = \int_{\sigma(0)}^{\sigma(T)} \frac{\varepsilon(\omega)^2}{a(\sigma^{-1}(\omega))\,\dot\sigma(\sigma^{-1}(\omega))}\dd\omega\,.
\end{align*}
Next, we specialize to the noise schedule with $a(t) = \ta \sigma(t)^2$, $\dot \sigma(t) = -(1-\ta)\,\sigma(t)$.
This yields
\begin{align*}
    \epsBD^2
    &= -\int_{\sigma(0)}^{\sigma(T)} \frac{\eps(\omega)^2}{\ta\, \omega^2\,(1-\ta)\,\omega}\,\dd\omega
    = \int_{\sigma(T)}^{\sigma(0)} \frac{1}{\ta\,(1-\ta)\, \omega^3\,\prior(\omega)}\,\eps(\omega)^2\dd\prior(\omega) \\
    &\le \Bigl(\max_{\omega \in [\sigma(T), \sigma(0)]} \frac{1}{\ta\, (1-\ta)\, \omega^3\, \prior(\omega)}\Bigr)\,\cE(\widehat f_\theta)\,.
\end{align*}
When $\prior(\omega) \propto \omega^{-3}$ on $[\sigma_T, \sigma_0]$, then $\epsBD^2 \propto \cE(\widehat f_\theta)$.

\section{Technical results}

\subsection{Calculations based on Tweedie's formula}

\begin{lemma}\label{lem:tweedie_formula}
The following hold:
\begin{align}
\omega^2\, \nabla \log p_\omega(y) &= f_\omega^\star(y) - y\,, \\
\omega^2\, \nabla^2 \log p_\omega(y) &= \omega^{-2}\, C_\omega(y) - I\,,\\
\omega^2\, \nabla \tr C_\omega(y)&=W_\omega(y) - 2C_\omega(y)\,(f_\omega^\star(y) - y)  \,. 
\end{align}
\end{lemma}
\begin{proof}
We only prove the third equality, as it is the least standard. Also, we will temporarily drop the conditioning variable and write $\E_\omega$ as shorthand for $\E_{q_\omega(\cdot|y)}$. 

We first notice that
\begin{align*}
    \E_\omega\|X-y\|^2 = \tr C_\omega(y) + \|f_\omega^\star(y) - y\|^2\,.
\end{align*}
Thus the expression for $W_\omega(y)$ simplifies to
\begin{align}\label{eq:W_expr}
    W_\omega(y)
    &=  
    \E_\omega[(X - f_\omega^\star(y))\,\|X-f_\omega^\star(y)\|^2] + 2 C_\omega(y)\,(f_\omega^\star(y) - y)\,.
\end{align}
We now want to show
\begin{align*}
    \nabla \tr C_\omega(y) = \omega^{-2}\,\E_\omega[(X-f_\omega^\star(y))\,\|X-f_\omega^\star(y)\|^2]\,.
\end{align*}
To this end, we compute (here note the gradients are with respect to $y$)
\begin{align*}
    \nabla_y \tr C_\omega(y) &= \nabla_y \int \|x-f_\omega^\star(y)\|^2 \dd q_\omega(x|y) \\
    &= \int 2\,\nabla_y f_\omega^\star(y)\,(x - f_\omega^\star(y)) \dd q_\omega(x|y) \\
    &\qquad{} + \int \|x-f_\omega^\star(y)\|^2\, \bigl(\nabla_y \log q_\omega(x|y)\bigr) \dd q_\omega(x|y)\\
    &= \int \|x-f_\omega^\star(y)\|^2\, \bigl(\nabla_y \log q_\omega(x|y)\bigr)\dd q_\omega(x|y)\,. 
\end{align*}
And finally we complete the claim by computing the following:
\begin{align*}
    \nabla_y \log q_\omega(x|y) &= \nabla_y \log p_X(x) - \frac{1}{2\omega^2}\,\nabla_y\|x-y\|^2 - \nabla \log Z_\omega(y) \\
    &= 0 + \frac{1}{\omega^2}\,(x-y) - \nabla_y \log \int p_X(x)\exp\bigl(-\frac{1}{2\omega^2}\,\|x-y\|^2\bigr) \dd x \\
    &= \frac{1}{\omega^2}\,(x-y) - \int \frac{\omega^{-2}\,(x-y)\, p_X(x)\exp\bigl(-\frac{1}{2\omega^2}\,\|x-y\|^2\bigr)}{Z_\omega(y)} \dd x\\
    &= \frac{1}{\omega^2}\,(x - y - f_\omega^\star(y) + y) = \frac{1}{\omega^2}\,(x-f_\omega^\star(y))\,,
\end{align*}
where $Z_\omega(y)$ is the normalizing constant of $q_\omega(\cdot|y)$.
\end{proof}

\subsection{Tools for proving the main results}

Letting $\Ndata(r)\defeq N(\cX; r, \norm{\cdot})$ be the covering number of $\cX$ under the Euclidean norm, recall that $$\dstar=1+\log \Ndata\bigl(\tfrac{\sigma_T^2}{\sigma_0\sqrt{d}}\bigr)\,.$$
We use $q_\sigma(\cdot|y)$ to denote the posterior of $X \sim p_X$ given $Y \sim \cN(X, \sigma^2 I)$, and
\begin{align*}
    D_{\sigma,\ell}(\bar x, y) \defeq \int \|x_0 - \bar x\|^\ell\dd q_{\sigma}(x_0|y)\,.
\end{align*}
Also, recall the definitions of $f_\sigma^\star$, $C_\sigma$, and $W_\sigma$ from~\eqref{eq:deff},~\eqref{eq:defC}, and~\eqref{eq:defW} respectively.

{The following proposition is adapted from~\citet[Proposition E.7]{Chen+26HighAcc}. To make the presentation more self-contained and because we have modified the statements, we provide a proof here.}

\begin{proposition}\label{prop:helper_prop}
For $\tsigma\geq0$, we write $\PP_{\tsigma}(\cdot)$ to be the probability measure under which $\xbar\sim \pdata$, $V\sim \cN(0,\tsigma^2\Id)$.

For any $\sigma, \tilde\sigma \in [\sigma_T,\sigma_0]$, $\ell \ge 1$, and $\delta\in(0,1)$, with probability at least $1-\delta$ under $\PP_{\tsigma}$,
\begin{align*}
    D_{\sigma,\ell}(\bar x, \bar x + V)
    &\lesssim_\ell \bigl\{(\sigma^2 + \tsigma^2)\,(\msf k + \log(1/\delta))\bigr\}^{\ell/2}\,, \\
    \sqrt{\langle V, C_\sigma(\bar x + V)\,V\rangle}
    &\lesssim (\sigma^2 + \tsigma^2)\,(\msf k + \log(1/\delta))\,,
\end{align*}
and
\begin{align}\label{eq:normdiff}
    |D_{\sigma,2}(\bar x+V,\bar x+V) - \tilde \sigma^2 d| \lesssim \tsigma^2 \sqrt{d\log(1/\delta)} +(\sigma^2 +\tsigma^2) \,(\mathsf k + \log(1/\delta))\,.
\end{align}
\end{proposition}

\begin{proof}
Let $r\defeq \frac{\sigma_T^2}{\sigma_0\sqrt{d}}$, and fix an $r$-covering of $\cX$, denoted $\cX \subseteq \bigcup_{i=1}^N B_i$, where $N\defeq N_X(r)$ and $B_i \defeq B(z_i,r)$ for $i = 1,\ldots,N$. Additionally, let $\cI \defeq \{i\in[N]: \pbar(B_i)\geq \alpha\}$ for some $\alpha > 0$ to be chosen later, and let $\cX_+ \defeq \bigcup_{i\in\cI} B_i$.

\textbf{Claim 1.} Define $M_1\defeq \tsigma\,(\sqrt{d}+2\sqrt{\log(2/\delta)})$, $M_2\defeq \tsigma \sqrt{2\log(4N/\delta)}$, and
\begin{align*}
    \cV \defeq \{V: \norm{V}\leq M_1,\, \abs{\langle V,x-x'\rangle }\leq M_2\,\|x-x'\|+2r\,(M_1+M_2),~\forall x,x'\in\cXbar\}\,.
\end{align*}
Then for any $\tsigma\ge 0$, it holds that $\PP_{\tilde\sigma}(V\not \in \cV)\leq \delta$.

\textbf{Proof of Claim 1.}
By Gaussian concentration, it is clear that for $t\geq 0$, $$\PP_{\tilde\sigma}(\norm{V}\geq \tsigma\,(\sqrt{d}+2\sqrt{t}))\leq e^{-t}\,.$$ 
Further, for any $i,j\in[N]$, it holds that $\PP_{\tilde\sigma}(\abs{\langle V,z_i-z_j\rangle}\leq \tsigma\norm{z_i-z_j}\sqrt{2t})\leq 2e^{-t}$ for $t\geq 0$. Hence, by a union bound, we can show that $\PP_{\tilde\sigma}(V\not \in\cV_0)\leq \delta$, where
\begin{align*}
    \cV_0=\{V: \norm{V}\leq M_1,\, \abs{\langle V,z_i-z_j\rangle}\leq M_2\,\|z_i-z_j\|,~\forall i,j\in[N]\}\,.
\end{align*}
Note that for any $V\in\cV_0$, it is clear that for any $i,j\in[N]$, $x\in B_i$, $x'\in B_j$, we can bound 
\begin{align*}
    \abs{\langle V,x-x'\rangle}
    &\leq \abs{\langle V,z_i-z_j\rangle}+2r\norm{V}
    \\
    &\leq M_2\norm{z_i-z_j}+2rM_1 
    \\
    &\leq M_2\,\|x-x'\|+2r\,(M_1+M_2)\,.
\end{align*}
This implies that $\cV_0\subseteq \cV$, and the proof of Claim 1 is hence completed.
\qed

Note that for any $z\in\RR^d$ and $M > 0$, $\ell \ge 1$,
\begin{align*}
    D_{\sigma,\ell}(\bar x, y)
    &\defeq \int \|x_0 - \bar x\|^\ell\dd q_{\sigma}(x_0|y)
    = \frac{\displaystyle \int \norm{x_0-\bar x}^\ell\exp\bigl(-\frac{\norm{y-x_0}^2}{2\sigma^2}\bigr)\dd p_X(x_0)}{\displaystyle \int \exp\bigl(-\frac{\norm{y-x_0}^2}{2\sigma^2}\bigr) \dd p_X(x_0)} \\[0.25em]
    &\lesssim_\ell M^\ell +\frac{\displaystyle \int (\norm{x_0-\bar x}-M)^{\ell}_+\exp\bigl(-\frac{\norm{y-x_0}^2}{2\sigma^2}\bigr) \dd p_X(x_0)}{\displaystyle \int \exp\bigl(-\frac{\norm{y-x_0}^2}{2\sigma^2}\bigr) \dd p_X(x_0)}\,.
\end{align*}
Here and throughout the proof, $\lesssim_\ell$ indicates that the bound holds up to a constant depending only on $\ell$.
Now specialize the above to $y = \bar x + V$ for $\xbar\in\cX_+$ and $V\in\cV$, and fix $M\geq 8\,(M_2 + \sqrt{r\,(M_1 + M_2)})$. Together, we have that
\begin{align*}
    &D_{\sigma,\ell}(\bar x, \bar x + V)\\
    &\quad \lesssim_\ell M^\ell +\frac{\displaystyle \int (\norm{x_0-\bar x}-M)^{\ell}_+\exp\bigl(-\frac{\norm{\bar x + V-x_0}^2}{2\sigma^2}\bigr) \dd p_X(x_0)}{\displaystyle\int \exp\bigl(-\frac{\norm{\bar x  +V-x_0}^2}{2\sigma^2}\bigr) \dd p_X(x_0)} \\[0.25em]
    &\quad= M^\ell +\frac{\displaystyle\int (\norm{x_0-\bar x}-M)^{\ell}_+\exp\bigl(-\frac{\norm{\bar x -x_0}^2 + 2\,\langle V,\bar x - x_0\rangle}{2\sigma^2}\bigr) \dd p_X(x_0)}{\displaystyle\int \exp\bigl(-\frac{\norm{\bar x -x_0}^2 + 2\,\langle V,\bar x -x_0\rangle}{2\sigma^2}\bigr) \dd p_X(x_0)}\,.
\end{align*}
Since $V\in\cV$, for any $\xbar,x_0\in\cXbar$ with $\norm{\xbar-x_0}\geq M$, we can bound
\begin{align*}
    \abs{\tri{V,\xbar-x_0}}
    \leq M_2\norm{\xbar-x_0}+2r\,(M_1+M_2)\leq \frac{1}{4}\norm{\xbar-x_0}^2,
\end{align*}
and hence in this case $\norm{\xbar-x_0}^2+2\,\tri{V,\xbar-x_0}\geq \frac{1}{2}\norm{\xbar-x_0}^2$. This immediately implies that 
\begin{align*}
    &\int \prn{\norm{x_0-\xbar}-M}_+^\ell \exp\bigl(-\frac{\norm{\xbar-x_0}^2+2\,\tri{V,\xbar-x_0}}{2\sigma^2}\bigr)\dd p_X(x_0) \\[0.25em]
    &\qquad\leq \int \prn{\norm{x_0-\xbar}-M}_+^\ell \exp\bigl(-\frac{\norm{\xbar-x_0}^2}{4\sigma^2}\bigr) \dd p_X(x_0) \\
    &\qquad\leq M^\ell\exp\bigl(-\frac{M^2}{4\sigma^2}\bigr)\,,
\end{align*}
where we use the fact that $u\mapsto u^{\ell/2} e^{-u}$ is decreasing for $u\geq \ell/2$, and assuming further that $M^2 \ge 2\ell\sigma^2$.

On the other hand, suppose that $\xbar\in B_i$ such that $\pbar(B_i)\geq \alpha$. Note that for any $x\in B_i$, we can bound
\begin{align*}
    \norm{\xbar-x_0}^2+2\,\tri{V,\xbar-x_0}
    &\leq 4r^2+4r\norm{V}
    \le 4r^2 + 4\tilde\sigma r\,(\sqrt d + 2\sqrt{\log(2/\delta)}) \\
    &\le 8\sigma^2\,(1+\sqrt{\log(2/\delta)})\,,
\end{align*}
by the choice of $r$.
This implies that
\begin{align*}
    \int \exp\bigl(-\frac{\norm{\xbar-x_0}^2+2\,\tri{V,\xbar-x_0}}{2\sigma^2}\bigr) \dd p_X(x_0)
    &\geq e^{-4-4\sqrt{\log(2/\delta)}}\,p_X(B_i)
    \geq c_0\alpha \delta\,,
\end{align*}
where $c_0>0$ is an absolute constant.

Combining the above bounds, we can conclude that as long as $\xbar\in\cX_+$ and $V\in\cV$, $M= 8\,(M_2+ \sqrt{r\,(M_1 + M_2)})+\sqrt{2\ell}\sigma + 2\sigma\sqrt{\log(1/(c_0\alpha\delta))}$, it holds that
\begin{align*}
    D_{\sigma,\ell}(\bar x,\bar x + V)
    \lesssim_\ell M^\ell+\frac{M^\ell}{c_0\alpha\delta}\exp\bigl(-\frac{M^2}{4\sigma^2}\bigr)\lesssim M^\ell\,.
\end{align*}
Note that $\PP_{\xbar\sim \pbar}(\xbar\not\in \cX_+)\leq \sum_{i\notin \cI} \pbar(B_i)\leq N\alpha=\delta$, where we choose $\alpha = \delta/N$. Hence, with probability at least $1-\delta$ under $\PP_{\tsigma}$,
\begin{align*}
    D_{\sigma,\ell}(\bar x, \bar x + V) \lesssim_\ell M^\ell\,.
\end{align*}

Similarly, we note that for $V\in\cV$,
\begin{align*}
    \langle V, C_\sigma(\xbar+V)\,V\rangle
    &\le \int \langle V, x_0 - \bar x\rangle^2\dd q_\sigma(x_0|\bar x + V) \\
    &\le \int (M_2\norm{x_0-\xbar}+2r\,(M_1+M_2))^2\dd q_\sigma(x_0|\xbar+V) \\
    \leq&~ 2M_2^2\,D_{\sigma,2}(\bar x, \bar x + V)+8r^2\,(M_1+M_2)^2\,.
\end{align*}
Our argument above then shows that with probability at least $1-\delta$ under $\PP_{\tsigma}$,
\begin{align*}
    \langle V, C_\sigma(\bar x + V)\,V\rangle \lesssim M^4\,.
\end{align*} 

For the final claim, note that 
\begin{align*}
    D_{\sigma,2}(y,y)
    &= \int\|y-x_0\|^2\dd q_\sigma(x_0|y) \\
    &= \|y-\bar x\|^2 + 2\,\Bigl\langle y-\bar x, \int (\bar x - x_0)\dd q_\sigma(x_0|y)\Bigr\rangle + D_{\sigma,2}(\bar x, y)\,.
\end{align*}
Setting $y=\bar x + V$,
\begin{align*}
    D_{\sigma,2}(\bar x +V,\bar x +V)
    &= \|V\|^2 - 2 \int \langle V, \bar x - x_0\rangle\dd q_\sigma(x_0|\bar x+V) + D_{\sigma,2}(\bar x, \bar x + V)\,.
\end{align*}
This readily yields, for $V \in \cV$,
\begin{align*}
    \bigl\lvert D_{\sigma,2}(\bar x + V,\bar x + V) - \|V\|^2\bigr\rvert
    &\lesssim M_2\int\|\bar x - x_0\|\dd q_\sigma(x_0|\bar x + V) + r\,(M_1 + M_2) + M^2 \\
    &= M_2\,D_{\sigma,1}(\bar x,\bar x + V) + r\,(M_1 + M_2) + M^2
    \lesssim M^2\,.
\end{align*}
Together with Gaussian concentration, the claim is complete.
\end{proof}

As a corollary, we obtain the following inequalities.

\begin{corollary}\label{cor:helper_cor}
If $X_t \sim p_{\sigma_t}$ and $\sigma\in [\sigma_T,\sigma_0]$, then with probability at least $1-\delta$,
\begin{align}
    \tr C_{\sigma}(X_t) &\lesssim (\sigma_t^2+\sigma^2)\, (\mathsf k + \log(1/\delta))\,, \nonumber\\
    \|W_\sigma(X_t)\|
    &\lesssim \bigl[(\sigma_t^2 + \sigma^2)\,(\msf k + \log(1/\delta))\bigr]^{3/2}\,. \nonumber
\end{align}    
\end{corollary}
\begin{proof}
    We apply Proposition~\ref{prop:helper_prop} with $\tilde \sigma = \sigma_t$, $\bar x = X_0$, and $V = X_t - X_0$.

    For the first statement, we note that
    \begin{align*}
        \tr C_\sigma(X_t)
        &= \int \|x_0 - f^\star_\sigma(X_t)\|^2\dd q_{\sigma}(x_0|X_t) \\
        &\le \int \|x_0 - X_0\|^2\dd q_{\sigma}(x_0|X_t)
        = D_{\sigma,2}(X_0, X_0 + V)
    \end{align*}
    so the result follows from Proposition~\ref{prop:helper_prop}.

    For the second statement, recall from~\eqref{eq:W_expr} that
    \begin{align*}
        W_\sigma(X_t)
        &= \int (x_0 - f_\sigma^\star(X_t))\,\|x_0 - f_\sigma^\star(X_t)\|^2\dd q_{\sigma}(x_0|X_t) + 2C_\sigma(X_t)\,(f_\sigma^\star(X_t) - X_t)\,.
    \end{align*}
    Hence,
    \begin{align*}
        \|W_\sigma(X_t)\|
        &\le \int \|x_0 - f_\sigma^\star(X_t)\|^3\dd q_{\sigma}(x_0|X_t) + 2\,\| C_\sigma(X_t)\,(f_\sigma^\star(X_t) - X_t)\| \\
        &\le 8D_{\sigma,3}(X_0, X_0+V) + 2\tr C_\sigma(X_t)\,\|f_\sigma^\star(X_t) - X_0\| + 2\,\|C_\sigma(X_t)\,V\| \\
        &\le 8D_{\sigma,3}(X_0, X_0 + V) + 2\tr C_\sigma(X_t)\,D_{\sigma,1}(X_0, X_0 + V) \\
        &\qquad{} + 2\sqrt{\tr C_\sigma(X_t)\,\langle V, C_\sigma(X_t)\,V\rangle} \\
        &\lesssim \bigl[(\sigma_t^2 + \sigma^2)\,(\msf k + \log(1/\delta))\bigr]^{3/2}\,. \qedhere
    \end{align*}
\end{proof}

\section{Experimental details}
\subsection{Analytical blind denoiser}
\label{app:gaussian}
In this section, we detail the results shown in \Cref{ssec:analytic_denoiser}. Additionally, we present more examples of sampling from arbitrary mixture of high dimensional Gaussian embedded in low dimensions. 

\subsubsection{Optimal denoiser for Gaussian Mixture}
Take the true distribution to be a mixture of Gaussians
$$ p(x) = \frac{1}{K}\sum_{i=1}^K \mathcal{N}(x; m_i , \Sigma_0 ) $$
If we add centered Gaussian noise with variance $\sigma^2$ to the data ($x_\sigma = x + \sigma z$), the distribution of noisy images will be
$$ p(x_\sigma) = \frac{1}{K}\sum_{i=1}^K \mathcal{N}(x_\sigma; m_i , \Sigma_0 + \sigma^2 I  ) $$
The score of this density is 
\begin{align}
    \nabla \log p(x_\sigma) =  - (\Sigma_0 + \sigma^2 I)^{-1} \sum_{i=1}^K \omega_i(x_\sigma)(x_\sigma - m_i) 
\end{align}
where $$\omega_i(x_\sigma) = \frac{ \mathcal{N}(x_\sigma; m_i , \Sigma_0 +  \sigma^2I)}{ \sum_{j=1}^K \mathcal{N}(x_\sigma; m_j , \Sigma_0 + \sigma^2I) }$$
So, the optimal denoiser for a noisy observation can be written in closed-form as  
\begin{align}
    f^{\star}(x_\sigma, \sigma, \{m_i\}_{i=1}^K, \Sigma_0) = x_\sigma - \sigma^2 \, (\Sigma_0 + \sigma^2 I)^{-1} \sum_{i=1}^K \omega_i(x_\sigma)\,(x_\sigma - m_i)\,.
\end{align}
To simulate the blind denoiser, we estimate the noise variance from a single noisy data point by maximizing the likelihood of noisy observation across $\log p_\sigma$. Thus, for a given noisy observation $y$, we compute $\hat{\sigma}$ such that
\begin{align*}
    \hat{\sigma} = \argmax_{\sigma} \log p_\sigma(y)\,.
\end{align*}
The \emph{blind} denoiser under maximum likelihood estimation (MLE) of the noise variance is therefore
\begin{align}
    f^{\dagger}(x_\sigma, \hat{\sigma}, \{m_i\}_{i=1}^K, \Sigma_0) = x_\sigma - \hat{\sigma}^2  \,(\Sigma_0 + \hat{\sigma}^2 I)^{-1} \sum_{i=1}^K \omega_i(x_\sigma)\,(x_\sigma - m_i)\,.
\end{align}
\begin{algorithm}[H]
  \caption{Sampling using an optimal blind denoiser with MLE noise variance}
  \label{alg:sampling-optimal-mle}
  \begin{algorithmic}
    \STATE {\bfseries Input:} Optimal denoiser $f^\dagger$, stepsize $h > 0$, diffusion coefficients $(a_t)_{t \in [0,T]}$, model parameters $\{m_i\}_{i=1}^K, \Sigma_0$, and $\sigma_0,\sigma_T > 0$
    \STATE {\bfseries Initialize} $X_0 \sim \cN(0, \sigma_0^2I)$, $k=0$
    \WHILE{ {\texttt{keepgoing} == \texttt{True}} }
    \STATE Compute $\hat{\sigma}_k = \text{argmax}_{\sigma} \log p_\sigma(X_{kh})$
    \STATE Compute $s_{k} = f^\dagger(X_{kh}, {\hat{\sigma}_k}, \{m_i\}_{i=1}^K, \Sigma_0) - X_{kh}$
    \IF{$\hat \sigma_k \leq \sigma_T$}
    \STATE \texttt{keepgoing} $\gets$ \texttt{False}
    \ELSE
    \STATE Update $X_{(k+1)h}\!\gets\! X_{kh} + hs_k +\sqrt{2}\int_{kh}^{(k+1)h}a_t \dd B_t$
    \ENDIF
    \STATE {Update $k \gets k+1$}
    \ENDWHILE
  \end{algorithmic}
\end{algorithm}

\subsubsection*{More examples}
\begin{figure}[H]
    \centering
   \includegraphics[width=.36\linewidth]{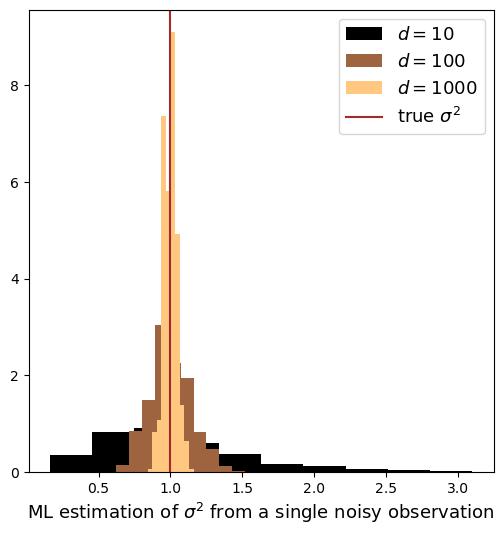}     
   \hfil
    \includegraphics[width=.39\linewidth]{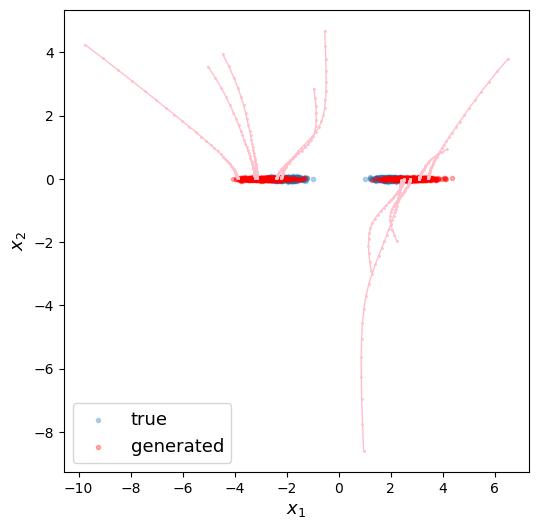}    
   \includegraphics[width=.36\linewidth]{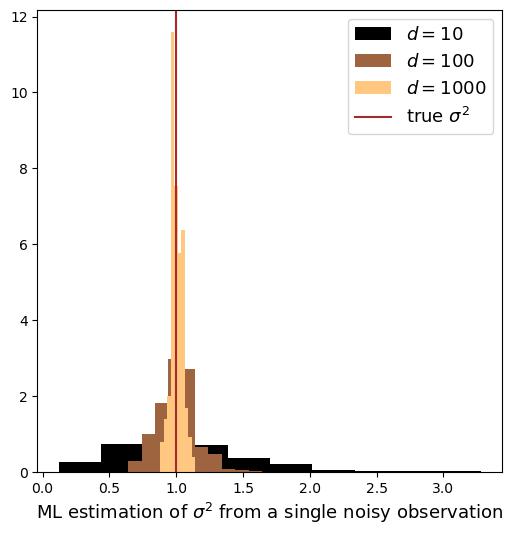}     
   \hfil
    \includegraphics[width=.39\linewidth]{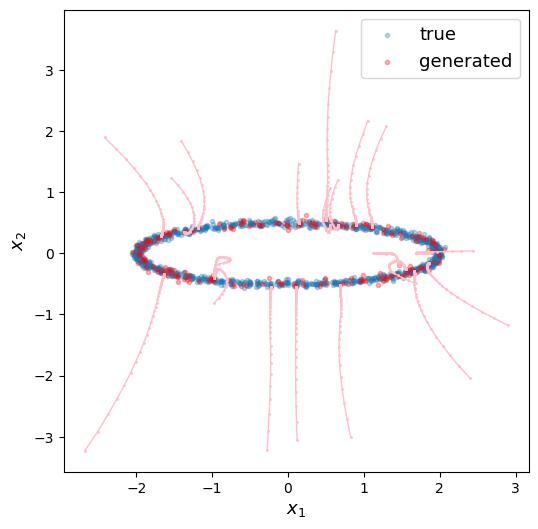}  
    
    \caption{{\textbf{Top:} Mixture of two Gaussians with covariances of rank 1}. Sampling is shown for $d = 1000$.  \textbf{Bottom:} Ellipse constructed as a mixture of Gaussians. Sampling is shown for $d = 100$.  
    }
    \label{fig:mixture_2Gaussian}
\end{figure}

\subsection{Blind denoisers trained on Gaussian and Gaussian mixture data}\label{app:gaussian_gaussian_mixture}
We present two sets of experiments on toy data that elucidate the impact of intrinsic dimensionality and blind denoising. In both cases, we use a $6$-layer MLP with width 512 and ReLU activations. We use the default AdamW optimizer in PyTorch, with learning rate equal to $10^{-3}$, and train for $50$k steps with batch size equal to 512. We take $\sigma_T = 0.01$ and $\sigma_0 = 10$. Additionally, $\prior$ was chosen to be a log-uniform prior on $[\sigma_T,\sigma_0]$.

\subsubsection*{Gaussian data}
We randomly generate a $\msf k \times \msf k$ covariance $\Sigma$ matrix, with $\msf k = 2$. We generate $n = 5000$ points and train two blind denoisers: $\widehat f_\theta^{(2)}$ and $\widehat f_\theta^{(100)}$. In the first scenario, the input dimension is the same as the intrinsic dimension of the data. In the latter, we pad the data with zeros until the data dimension is $d=100$, thus changing the input dimension of the network. 

Figure~\ref{fig:learning_gaussian_denoising} displays the output from Algorithm~\ref{alg:inf} using an exponential Euler integrator with $h = 0.5$ and $a_t = 0.5 \sigma_t^2$. We see that the generated samples for $\msf k = d = 2$ are heavily concentrated and do not accurately capture the training data. Moreover, the noise schedule is both too fast at the beginning and too slow at inference. In contrast, the blind denoiser trained with larger input dimension manages to both (i) generate viable samples, and (ii) obey the noise schedule from Proposition~\ref{prop:opt_schedule} exactly. 
\subsubsection*{Gaussian-mixture data}
We repeat the same experiment on $2$-component Gaussian mixture data using randomly generated means and covariances with equal weight with $\msf k = 2$. We use $n = 50 000$ data points and additionally incorporated weight-decay in the AdamW optimizer of value $5 \times 10^{-2}$. All other parameters remain the same, except now $d \in \{2,200\}$. The behavior is essentially identical to Figure~\ref{fig:learning_gaussian_denoising}, where low intrinsic and high ambient dimensionalities are crucial for blind denoisers to learn and sample from the data distribution.
\begin{figure}[H]
    \centering
   \includegraphics[width=.45\linewidth]{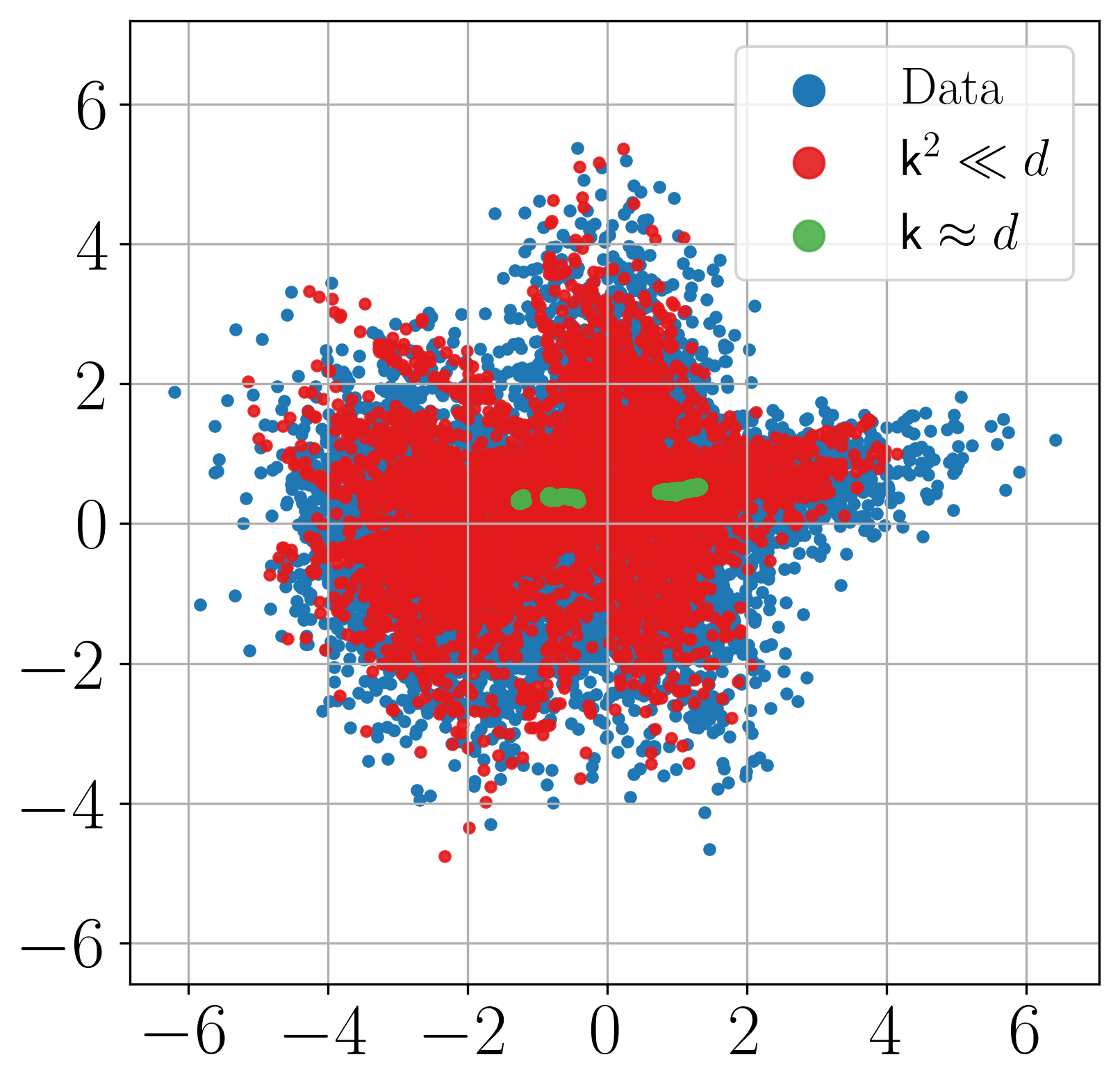}
\includegraphics[width=.45\linewidth]{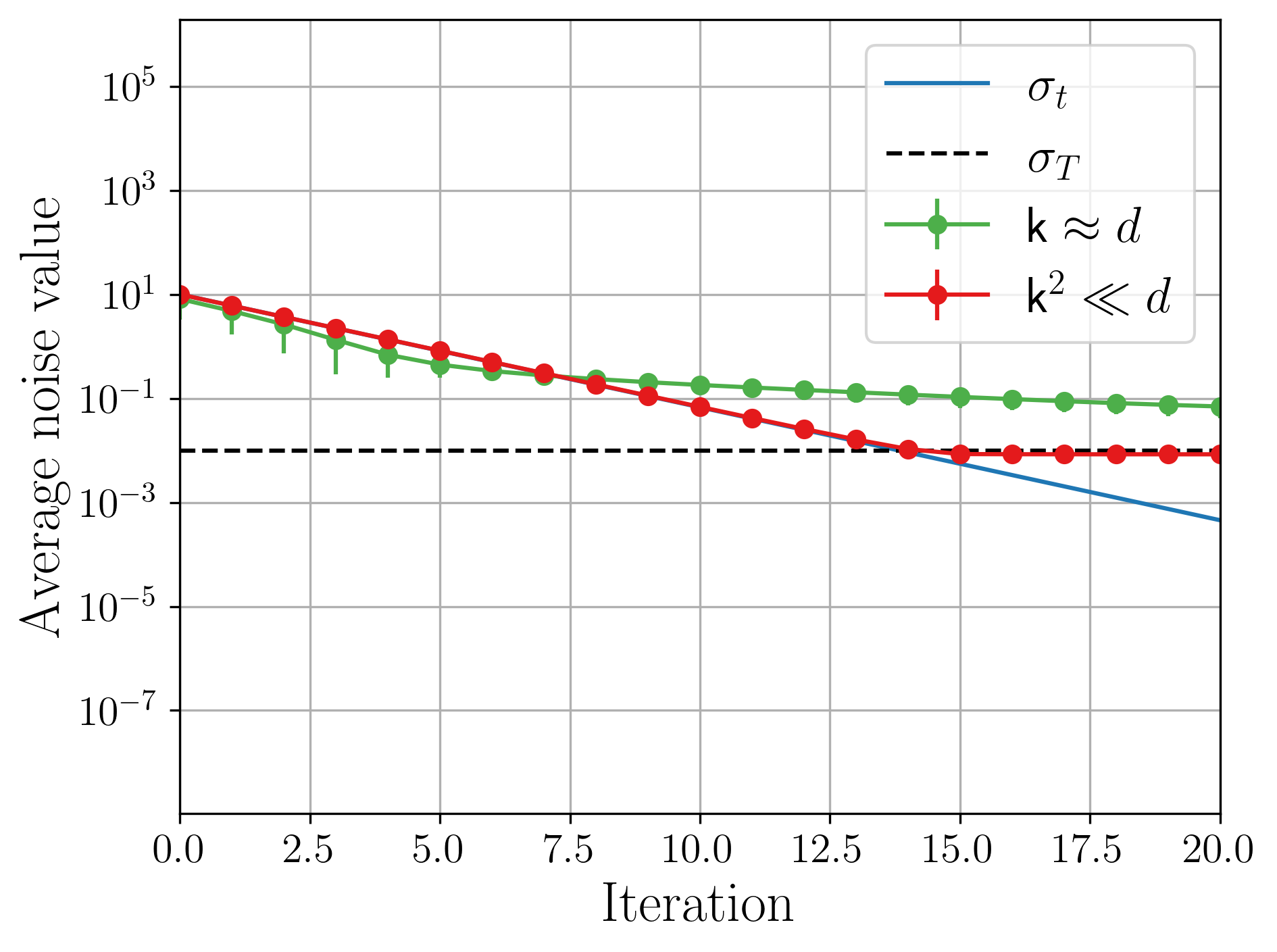}
    \caption{ Sampling performance of BDDMs trained on a Gaussian-mixture dataset with intrinsic dimension $\msf k =2$ and two different input dimensions $d \in \{2,200\}$.  Generated samples \textbf{(left)} and evolution of the estimated noise level, corresponding to Proposition~\ref{prop:opt_schedule} \textbf{(right)}. }
\label{fig:learning_gaussianmixture_denoising}
\end{figure}

\subsection{Photographic images}
\label{app:images}

\subsubsection*{Architecture} 
The U-Net is the most widely used architecture for denoising and score estimation. We adopt the vanilla U-Net \citep{ronneberger2015u}, composed of convolutional layers with ReLU nonlinearities and layer normalization, organized into downsampling and upsampling blocks. In more detail, there are 3 downsampling and 3 upsampling blocks with increasing number of channels [64, 128, 256]. The middle block has 512 channels. There are also 3 skip connections that carry details from each encoder to the corresponding decoder at the same scale. The number of layers in the encoding blocks is 2 and decoding blocks is 6. 

For non-blind models, we use the standard noise-level embedding: the noise level is mapped through a Fourier feature transform, processed by an MLP, and injected into all normalization layers via scale and shift parameters.

Our goal is to match architectural and training hyperparameters as closely as possible (except for noise conditioning versus blindness) to enable fair comparisons, so all models are trained from scratch. The architecture is intentionally small (\textbf{around 13 million parameters}) and simple (e.g., no attention or text conditioning) compared to commonly used models with hundreds of millions of parameters; the goal is not SOTA performance but a controlled comparison of blind versus non-blind models.

\subsubsection*{Training} 
Blind models are trained by minimizing the empirical loss in \eqref{eq:empirical_loss}, following \Cref{alg:train}, while non-blind models are trained using the noise-conditioned counterpart. The noise level $\sigma \in (0, 3]$ is sampled from a density proportional to $1/\sqrt{\sigma}$. We use a batch size of 512 for all models. The learning rate is 0.001 with a rate of decay of 100. We fix the total number of epochs (i.e., full pass over the dataset) to 1000. Training is done on four H100 NVIDIA GPUs per model, taking around 24 hours for CelebA and 34 hours for Bedroom datasets. 

We observe diminishing mean-square error improvements for widening range of $\sigma$ when the model is trained to predict the clean image; this generalization beyond the training noise range does not hold for models trained to predict the residual directly. This suggests that neural networks more readily represent constants (e.g., zero or the data mean) than the identity mapping.

\subsubsection*{Noise level estimator model}
We train a small network to estimate the noise level from a given noisy image. Given the simplicity of this task, our network consists of 4 convolutional layers with filter size of $3\times3$, and the number of channels in the intermediate layers is 64. To match the denoisers, we choose $\sigma \in (0,3]$. The trained model can estimate the noise level almost perfectly, as seen in Figure~\ref{fig:noise-estimator-model}.

\begin{figure}[H]
    \centering
    \includegraphics[width=.5\linewidth]{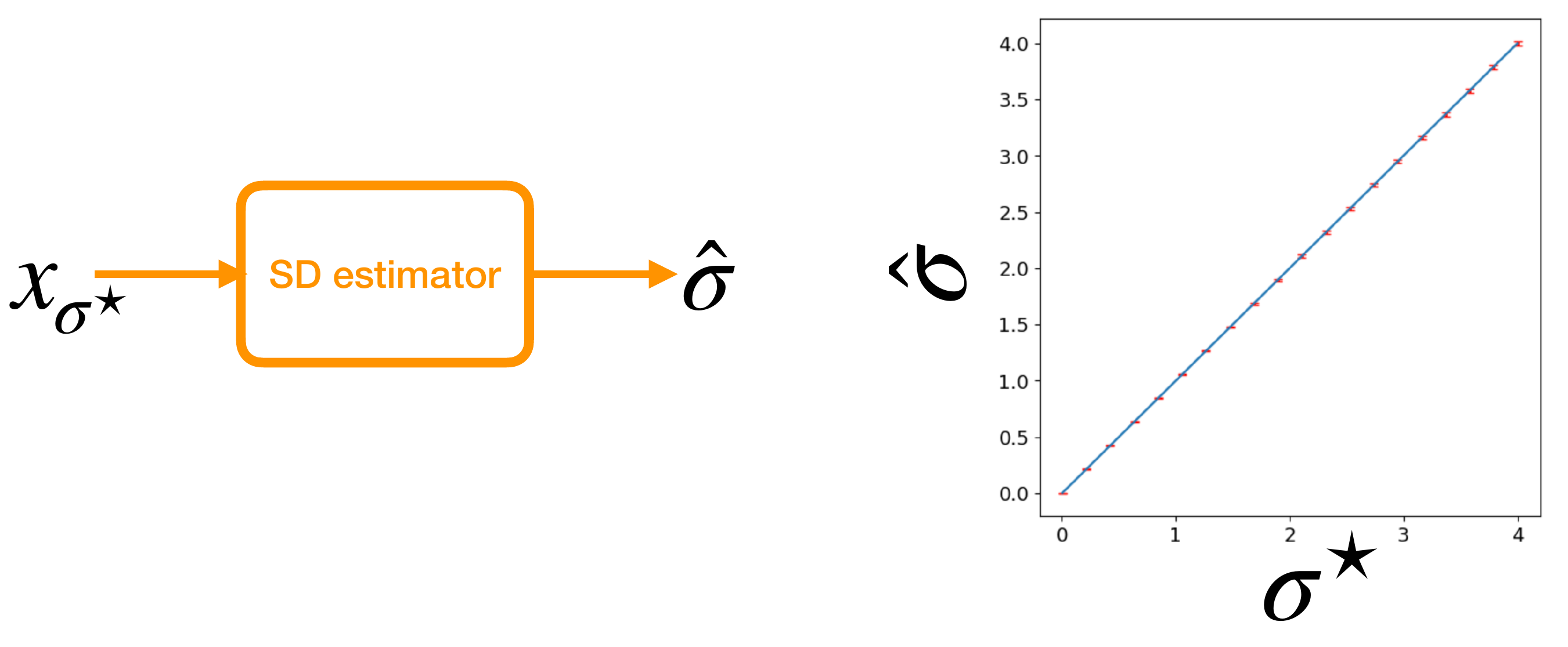}       
    \caption{A smaller neural network trained to estimate $\sigma$ from noisy image is very accurate. Error bars show two standard deviations. It is also worth noting that the model performs well beyond training range $\sigma \in (0,3]$. 
    }
    \label{fig:noise-estimator-model}
\end{figure}
\subsubsection*{More sampling examples}
\label{app:more_samples}

\begin{figure}[!htb]
    \centering
   \includegraphics[width=.9\linewidth]{figures/example_train_celeb.pdf}  
   \includegraphics[width=.9\linewidth]{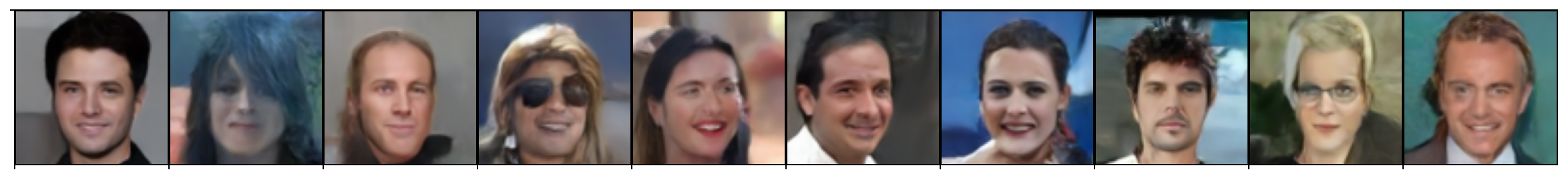}    
   \includegraphics[width=.9\linewidth]{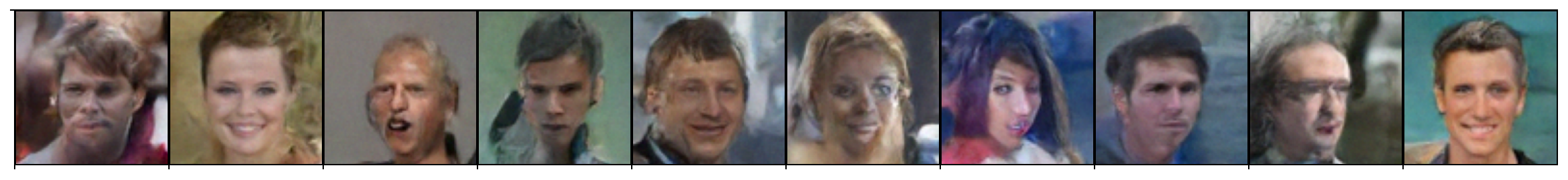}      
    \caption{Continued from \Cref{fig:samples_celeba_1k} with \textbf{lower} number of steps: 
    $N \approx 100$ for BDDM shown in second row and $N = 100$ for VE-DDPM shown in third row. The samples from BDDM are still of higher quality than samples from DDPM. Samples in each column are initialized with the same seed. 
    }
    \label{fig:samples_celeba_100}
\end{figure}

\begin{figure}[!htb]
    \centering
   \includegraphics[width=.9\linewidth]{figures/example_train_celeb.pdf}  
   \includegraphics[width=.9\linewidth]{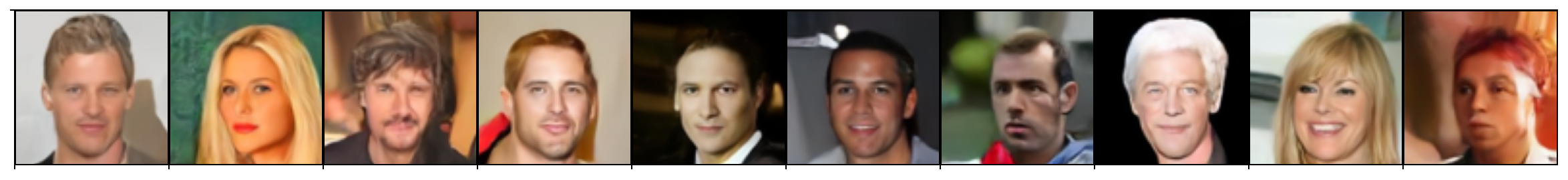}    
   \includegraphics[width=.9\linewidth]{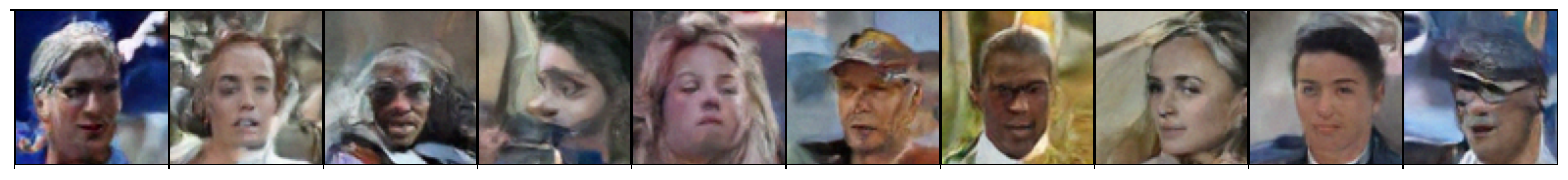}      
    \caption{Continued from \Cref{fig:samples_celeba_1k} with \textbf{higher} number of steps: 
    $N \approx 17,000$ for BDDM shown in second row and $N = 17,000$ for VE-DDPM shown in third row. The samples from BDDM are still of significantly higher quality than samples from DDPM. Samples in each column are initialized with the same seed. 
    }
    \label{fig:samples_celeba_17k}
\end{figure}

\begin{figure}[!htb]
    \centering
   \includegraphics[width=.9\linewidth]{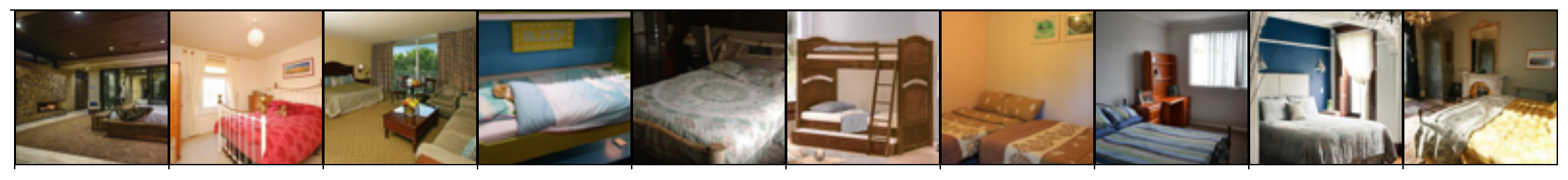}  
   \includegraphics[width=.9\linewidth]{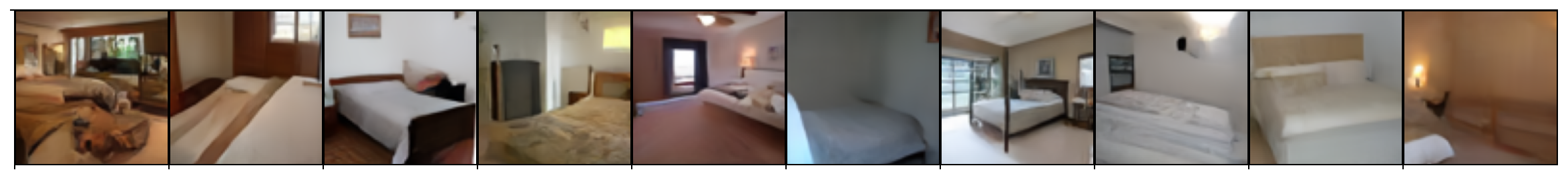}           
   \includegraphics[width=.9\linewidth]{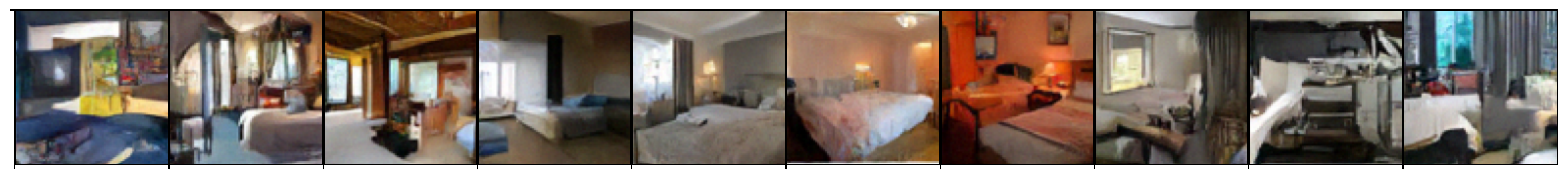}     
    \caption{\textbf{Top row:} example training images from the LSUN dataset. \textbf{Second row:} samples generated by BDDM with average number of steps $N \approx 1300$. \textbf{Third row:} Samples generated by DDPM (VE) algorithm, with total number of steps $N = 1300$. Samples in each column are initialized with the same random seed with matched injected noise. 
    }
    \label{fig:samples_bed}
\end{figure}

\end{document}